%% file: main-rpt.tex
\documentclass{article}
\usepackage{graphicx} 
\usepackage{caption}
\usepackage{subcaption}
\usepackage{algorithm}
\usepackage{algorithmic}
\usepackage{multicol}
\usepackage{fullpage}
\usepackage{enumitem}
\usepackage{hyperref}

\graphicspath{ {./fig/network/} }
\graphicspath{ {./fig/PET/} }
\graphicspath{ {./fig/temporalRS/} }
\input{macros}

\begin{document} 

\title{Fast and Simple Optimization for Poisson Likelihood Models
\thanks{Part of this work was done while NH was at Georgia Tech and ZH was at Inria. 
The authors would like to thank Anatoli Juditsky, Julien Mairal, Arkadi Nemirovski, and Joseph Salmon for fruitful discussions. 
This work was supported by the NSF Grant CMMI-1232623, the project Titan (CNRS-Mastodons), the MSR-Inria joint centre, the LabEx Persyval-Lab (ANR-11-LABX-0025), 
the project Macaron (ANR-14-CE23-0003-01), and the program ``Learning in Machines and Brains'' (CIFAR).}}
\author{Niao He \\
\texttt{niaohe@illinois.edu} \\
University of Illinois Urbana-Champaign 
\and
Zaid Harchaoui \\
\texttt{zaid.harchaoui@nyu.edu} \\
Courant Institute, New York University
\and 
Yichen Wang\\
\texttt{yichen.wang@gatech.edu} \\
Georgia Institute of Technology
\and
Le Song\\
\texttt{lsong@cc.gatech.edu} \\
Georgia Institute of Technology
}

\maketitle

%%%%% Abstract %%%%%%%
\begin{abstract} 
Poisson likelihood models have been prevalently used in imaging, social networks, and time series analysis. 
We propose fast, simple, theoretically-grounded, and versatile, optimization algorithms for Poisson likelihood modeling. 
The Poisson log-likelihood is concave but \emph{not Lipschitz-continuous}. Since almost 
all gradient-based optimization algorithms rely on Lipschitz-continuity, optimizing Poisson likelihood models
with a guarantee of convergence can be challenging, especially for large-scale problems. 

We present a new perspective allowing to efficiently optimize a wide range of penalized Poisson likelihood objectives. 
We show that an appropriate saddle point reformulation enjoys a favorable geometry and a smooth structure. 
Therefore, we can design a new gradient-based optimization algorithm with $O(1/t)$ convergence rate, in contrast to the usual $O(1/\sqrt{t})$ rate of non-smooth minimization alternatives. Furthermore, in order to tackle problems with large samples, we also develop a randomized block-decomposition variant that enjoys the same convergence rate yet more efficient iteration cost. 

Experimental results on several point process applications including social network estimation and temporal recommendation 
show that the proposed algorithm and its randomized block variant outperform existing methods both on synthetic and real-world datasets.  
\end{abstract}

%%%%%
\input{sec1.tex}

\input{sec2.tex}

\input{sec3.tex}
\input{sec4.tex}
\input{sec5.tex}

\input{sec6.tex}
%%%%%

%%%%%  References %%%%%%% 
%\setlength{\bibsep}{1.5pt}
%{\small
%{
%\bibliography{references,bibfile}}
%\bibliographystyle{abbrv}
%}

%\setlength{\bibsep}{1.5pt}
\bibliographystyle{abbrv}
%\small{

%}

\newpage
\appendix
\onecolumn
\input{appendix.tex}
\end{document}

%% file: macros.tex
\usepackage{amsthm,amsmath,amssymb,amsfonts}
\def\argmin{\mathop{\hbox{\rm  argmin}}}
\def\Argmin{\mathop{\hbox{\rm  Argmin}}}

\def\bR{\mathbf{R}}
\def\bE{\mathbf{E}}

\def\cX{{\cal X}}

\def\cL{{\cal L}}

\def\Opt{{\hbox{\rm  Opt}}}

\def\Tr{{\hbox{\rm  Tr}}}

\def\epsilonvi{{\epsilon_{\hbox{\scriptsize\rm VI}}}}

\def\nuc{{\hbox{\scriptsize\rm nuc}}}

\def\cT{{\cal T}}

\def\mybox\blacksquare
\def\Bbb#1{{\bf #1}}

\def\fnote#1{\footnote}
\def\blacksquare{\hbox{\vrule width 4pt height 4pt depth 0pt}}

\def\beq{\begin{equation}}
\def\eeq{\end{equation}}

\newtheorem{theorem}{Theorem}[section]
\newtheorem{lemma}{Lemma}[section]

\newtheorem{proposition}{Proposition}[section]

\def\Tr{\mathop{{\rm Tr}}}

\def\argmin{\mathop{\rm argmin}}

\def\log{\mathop{{\rm log}}}

\def\Argmin{\mathop{\rm Argmin}}

\def\Tr{\mathop{{\rm Tr}}}

\def\argmin{\mathop{\rm argmin}}

\def\log{\mathop{{\rm log}}}

\def\Argmin{\mathop{\rm Argmin}}

\def\hat{\widehat}

\def\Bbb#1{{\bf #1}}

\def\fnote#1{\footnote}
\def\blacksquare{\hbox{\vrule width 4pt height 4pt depth 0pt}}

\def\beq{\begin{equation}}
\def\eeq{\end{equation}}

\def\ba{\begin{array}}
\def\ea{\end{array}}
\def\beann{\begin{eqnarray*}}
\def\eeann{\end{eqnarray*}}
\def\bea{\begin{eqnarray}}
\def\eea{\end{eqnarray}}
\def\Tr{\mathop{{\rm Trace}}}

\def\argmin{\mathop{\rm argmin}}

\def\log{\mathop{{\rm log}}}

\def\Argmin{\mathop{\rm Argmin}}

\def\cX{{\cal X}}

\def\Bbb#1{{\bf #1}}

\def\fnote#1{\footnote}
\def\blacksquare{\hbox{\vrule width 4pt height 4pt depth 0pt}}

\def\Tr{\mathop{{\rm Tr}}}

\def\argmin{\mathop{\rm argmin}}

\def\Argmin{\mathop{\rm Argmin}}

\def\tilde{\widetilde}

\def\epsilonvi{{\epsilon_{\hbox{\scriptsize\rm VI}}}}

\def\argmin{\mathop{\hbox{\rm  argmin}}}
\def\Argmin{\mathop{\hbox{\rm  Argmin}}}

\def\Bbb#1{{\bf #1}}
\def\beq{\begin{equation}}

\def\blacksquare{\hbox{\vrule width 4pt height 4pt depth 0pt}}
\def\bR{{\mathbb{R}}}
\def\eeq{\end{equation}}

\def\fnote#1{\footnote}
\def\log{{\hbox{\rm  log}}}
\def\mypict3{\epsfxsize=220pt\epsfysize=80pt\epsfbox}

\def\Opt{{\hbox{\rm  Opt}}}

\def\Tr{{\hbox{\rm  Tr}}}
\def\bE{{{\mathbf{E}}}}

\def\cL{{{\cal L}}}

\def\bR{{\mathbf{R}}}

\def\nuc{{\hbox{\scriptsize\rm nuc}}} 
\def\Prox{\hbox{\rm Prox}}

\newcommand{\be}{\begin{eqnarray}}
\newcommand{\ee}[1]{\label{#1}\end{eqnarray}}

\newcommand{\ese}{\end{eqnarray*}}
\newcommand{\bse}{\begin{eqnarray*}}
\newcommand{\rf}[1]{(\ref{#1})}
\newcommand{\half}{ \mbox{\small$\frac{1}{2}$}}

\def\bE{{\mathbf{E}}}

%% file: sec1.tex
\section{Introduction}
We consider {\sl penalized Poisson likelihood} models~\cite{Green:Silverman:1994}, that are, models where the linear measurements are contaminated with Poisson-distributed noise $b_i\sim \text{Poisson}(a_i^Tx)$. 
Maximum likelihood estimators of such models are solutions to convex optimization of the form
\begin{equation}\label{model:problemofinterest0}
\min_{x\in \mathcal{X}} \: L(x)+h(x) :=\sum_{i=1}^m \left(a_i^Tx-b_i\log(a_i^Tx)\right)+h(x) 
\end{equation}
where $\mathcal{X}$ is the domain such that the objective is well-defined, $m$ is the number of observations, $-L(x)$ is the log-likelihood, and $h(x)$ is a regularization penalty. The latter can often be \emph{non-smooth} in order to promote desired properties of the solution. Popular examples include the $\ell_1$-norm, that enforces sparsity, or the nuclear-norm, that enforces low-rank structure.

Penalized Poisson likelihood models are popular for the study of diffusion networks  \cite{rajaram2005poisson, simma2012modeling,zhou2013learning,iwata2013discovering} and various time series problems \cite{gunawardana2011model,Nan15}, where cascading events are assumed to be triggered from some temporal point process. A widely used option is the self-exciting {\sl Hawkes process} \cite{hawkes1971spectra}. Given a sequence of events $\{t_i\}_{i=1}^m$ from such a point process with conditional intensity $\lambda(t)$, the negative log-likelihood  is $L(\lambda)=\int_{0}^T\lambda(t)dt-\sum_{i=1}^m\log(\lambda(t_i))$, which enjoys the structure in (\ref{model:problemofinterest0}) as long as 
$\lambda(t)$ is linear w.r.t to the learning parameters.  
We give a couple of interesting emerging examples.

\begin{itemize}%[leftmargin=*,nosep,nolistsep]
\item {\sl Network estimation.} Discovering the latent influences among social communities ~\cite{MohShoMarBraetal11, zhou2013learning} has been active research topic in the last decade. Given a sequence of events $\{(u_i,t_i)\}_{i=1}^m$, \cite{zhou2013learning} shows that the hidden network of social influences can be learned by solving the convex optimization:
\begin{equation}\label{model:plainHawkes}
\begin{array}{c}
\min_{x\geq 0, X\geq 0} L(\lambda(x,X)) +\lambda_1\|X\|_1+\lambda_2\|X\|_\nuc
\end{array}
\end{equation}
where  $x$ stands for the base intensity for all users and  $X$ the infectivity matrix, and the conditional intensity $\lambda(x,X|t_i)= x_{u_i}+\sum_{k:t_k<t_i}X_{{u_i}{u_k}}g(t_i-t_k)$ is linear in $(x,X)$, $g$ is triggering kernel.\\

\item {\sl Temporal recommendation system.} Incorporating temporal behaviors of customers into recommendation systems
has been studied in \cite{KapSubKarSriJaiSch15,Nan15} to improve personalized suggestions.  Given a sequence of events $\{\cT^{u,i}\}_{u,i}$ for each user-item pair $(u,i)$, to capture the recurrent temporal patterns,  \cite{Nan15} introduces an optimization problem with low-rank penalties:
\begin{equation}\label{model:temporalRS}
\begin{array}{c}
\min_{X_1\geq0, X_2\geq 0} \;L(\lambda(X_1,X_2))+\lambda_1\|X_1\|_\nuc+\lambda_2\|X_2\|_\nuc
\end{array}
\end{equation}
where $X_1$ and $X_2$ denote the base intensity and self-exciting coefficients for all user-item pair and intensity $\lambda(X_1,X_2|t_j\in \cT^{u,i})= X_1^{u,i}+X_2^{u,i}\sum_{t_k\in \cT^{u,i}: t_k<t_j}g(t_j-t_k)$ is linear in $(X_1,X_2)$. 
\end{itemize}

However, while such models are central and widespread in many real-world applications, there has been, in contrast, little research on designing efficient algorithms with theoretical guarantees to optimize the corresponding penalized likelihood objectives. Despite a significant body of work on efficient gradient-based (\textit{a.k.a.} first-order) methods for optimizing penalized likelihood models, ranging from proximal algorithms \cite{NesCompMin,Teboulle09a} to stochastic and incremental algorithms \cite{Nem09, bertsekas2011incremental, schmidt2013minimizing}, the overwhelming majority of works assumes that the log-likelihood is globally Lipschitz-continuous. 

Unfortunately, in the case of Poisson models, the log-likelihood is known to be  non-globally Lipschitz continuous or differentiable. Therefore, applying such algorithms to optimize penalized Poisson likelihoods is a rather heuristic approach, which may lead to disappointing results, as we shall show~in Sec.~\ref{sec:application}. There is an urgent need for new optimization algorithms, with theoretical guarantees, that can handle both the non-Lipschitzness of the Poisson likelihood and the potential non-smoothness of the regularization penalty. 

Another bottleneck of solving Poisson likelihood models, especially in the case of estimating point processes, is that computing the gradient requires almost the entire data. Since the observations are no longer independent, traditional stochastic gradient methods \cite{Nem09} simply would not work. It remains very challenging to efficiently learn point processes in the large-scale regime. 

\paragraph{Related work.}  Few works have addressed efficient optimization with non-Lipschitz objectives. An early motivating real-world problem was Poisson imaging reconstruction, hence these works mainly focused on this particular application. In~\cite{harmany2012spiral}, the authors propose to add a tolerance $\epsilon$ to each logarithmic term, which results in a smooth problem that comes with huge Lipschitz constant $L\sim O(1/\epsilon^2)$. In~\cite{sra2008non}, additional constraints $a_i^Tx\geq\epsilon, \forall i$ are added, which lead to computationally expensive projection steps. Another approach is explored in~\cite{tran2013composite}, where the authors exploit the self-concordance nature of the logarithmic term and propose a very sophisticated proximal gradient method, yet only with locally linear convergence. In a different line of work, e.g. in~\cite{ben2001ordered}, the authors treat this problem as general non-smooth minimization with Mirror Descent, which avoids the dependence on Lipschitz continuity of the gradient, but in the sacrifice of having a worse rate of convergence, i.e. $O(1/\sqrt{t})$. None of the above mentioned algorithms is efficient for the general purpose of solving Poisson likelihood models. After completing this work, we became aware of~\cite{yanez2014primal}, where a primal-dual algorithm similar to our Algorithm~\ref{CMPforPoisson} is proposed, for non-negative matrix factorization with Kullback-Leibler divergence.

\begin{table*}
\begin{center}
{\small
\text{}\\[-0.9cm]
\caption{Convergence rates of different algorithms for penalized Poisson likelihood models}
\begin{tabular}{|l|c|c|c|l|}
\hline
Optimization algorithm &type & guarantee  &convergence& constant\\
\hline\hline
Mirror Descent \cite{ben2001ordered, duchi2010composite} &batch&  primal  &$O(M/\sqrt{t})$ &  $M$ unbounded\\
\hline
Accelerated gradient descent \cite{harmany2012spiral,NesCompMin}& batch & primal   &$O(L/t^2)$ &  $L$  unbounded\\
\hline
Composite Mirror Prox (this paper) & batch & primal \& dual &  $O(\cL/t)$ &  $\cL$ bounded\\
\hline
Randomized Block Mirror Prox(this paper) &stoch. &sad. point gap & $O(\cL/t)$ &  $\cL$ bounded\\
\hline
\end{tabular}}
\end{center}
\end{table*}

\paragraph{Main contributions.} We propose a family of optimization algorithms that gracefully handles the non-Lipschitzness of the Poisson likelihood. The proposed algorithms hinge upon a novel saddle point representation of the Poisson likelihood objective, allowing us to circumvent the usual Lipschitz-continuity conditions pervasive in first-order algorithms. Our basic algorithm enjoys a $O(1/t)$ convergence rate in theory, in contrast to the typical $O(1/\sqrt{t})$ of the non-smooth minimization alternative~\cite{ben2001ordered}. To tackle large-scale problems, we also, for the first time, develop an extension of our basic algorithm, which can be seen as a randomized block-decomposition variant of Mirror Prox. This latter algorithm exhibits the same convergence rate yet with cheaper iteration cost. We provide extensive experimental results showing the algorithms can be used to efficiently estimate point processes from large amounts of data for resp. social network estimation and temporal recommendation. Results show the strong empirical performance of the basic algorithm and its large-scale extension compared to existing methods.  We will make the code available online upon publication.

%% file: sec2.tex
\section{Penalized Poisson Regression and Saddle Point Reformulation}\label{sec:saddlepoint}
 \paragraph{Problem statement.} We consider the following problem in a slightly more compact form:
\begin{equation}\label{model:problemofinterest}
\min_{x\in \bR_+^n} L(x)+h(x), \text{ with } L(x)=s^Tx-\sum\nolimits_{i=1}^m c_i\log(a_i^Tx)
\end{equation}
given coefficients $a_i, s\in\bR_+^n, c\in\bR_+^m, i=1,\ldots m$.  
\paragraph{Assumptions}
Define  proximal operator, $\Prox_{x_0}^{h}(\xi):=\argmin\nolimits_{x\in \bR_+^n}\{V_\omega(x,x_0)+\langle \xi,x\rangle +h(x)\},$ where the Bregman distance $V_\omega(x,x_0):=\omega_x(x)-\omega_x(x_0)-\nabla \omega_x(x_0)^T(x-x_0)$ is defined by distance generating function $\omega_x(\cdot)$ that is {\sl compatible} (i.e. Lipschitz continuous and 1-strongly convex) w.r.t. some norm $\|\cdot\|_x$ defined on $\bR^n$. Throughout the paper, we shall assume that 1) \emph{the regularized penalty $h$ is homogeneous}, \textit{i.e.} for any $a\in\bR, h(ax)=|a|h(x)$, and that 2) \emph{the proximal operator is proximal-friendly}, \textit{i.e.} can be computed in closed-form. 

Note that the above assumptions hold true for many sparsity-promoting penalty functions, including $h(x)=\|x\|_1$ and the ones considered in Sec.~\ref{sec:application}. See~\cite{bach-et-al,Bauschke:2011} for a survey of proximal operators in machine learning and signal processing.

 \paragraph{Saddle point reformulation} The crux of our approach is to utilize the Fenchel representation of the log function
\begin{equation}
\begin{array}{c}
\log(u)=\min_{v>0}\{uv-\log(v)-1\}.
\end{array}
\end{equation}
Hence, we can  rewrite (\ref{model:problemofinterest}) as
\begin{eqnarray*}
\min_{x\in \bR_+^n} \max_{v\in\bR_{++}^m} s^Tx+\sum\nolimits_{i=1}^m [c_i \log(v_i)-c_iv_ia_i^Tx+c_i]+h(x).
\end{eqnarray*}
Setting $y_i=c_iv_i$, this can be further simplified to
\begin{eqnarray}\label{model:problemsimplified}
\min_{x\in \bR_+^n} \max_{y\in\bR_{++}^m} s^Tx-y^TAx+\sum\nolimits_{i=1}^m c_i \log(y_i)+h(x)+c_0
\end{eqnarray}
where the matrix $A=[a_1^T;a_2^T;\ldots;a_m^T]$ and $c_0=\sum_{i=1}^m c_i(1-\log(c_i))$ is a constant. Observe that the above model can be regarded as a composite saddle point problem with two separable penalty functions -- a convex penalty $h(x)$ for variable $x$ and a concave penalty $p(y)=\sum_{i=1}^m c_i \log(y_i)$ for variable $y$. We first introduce some preliminary results for solving composite saddle point problems. 

\section{Notations and Preliminaries}\label{sec:preliminary}
In this section, we introduce some key concepts related to our setup and analysis.
\paragraph{Composite saddle point problem} Consider the convex-concave saddle point problem
\begin{equation}\label{compositeSP}
\min_{u_1\in U_1}\max_{u_2\in U_2}\Phi(u_1,u_2) :=\left[\phi(u_1,u_2)+\Psi_1(u_1)-\Psi_2(u_2)\right]
\end{equation}
under the situation
\begin{itemize}
\item $U_1\subset E_1$ and $U_2\subset E_2$ are nonempty closed convex sets in Euclidean spaces $E_1,E_2$;
\item $\phi(u_1,u_2)$ is a convex-concave function on $U_1\times U_2$ with  Lipschitz continuous gradient;
\item $\Psi_1:U_1\to\bR$ and $\Psi_2:U_2\to\bR$ are convex functions, perhaps nonsmooth, but ``fitting'' the domains $U_1$, $U_2$ in the following
sense: for $i=1,2$, we can equip $E_i$ with a norm $\|\cdot\|_{(i)}$, and $U_i$  with a compatible  distance generating function (d.g.f.)
$\omega_i(\cdot)$ in a way that subproblems of the form are easy to solve for any $\alpha>0,\beta>0$ and input $\xi\in E_i$
\begin{equation}\label{auxil}
\min\nolimits_{u_i\in U_i}\left\{\alpha \omega_i(u_i)+\langle \xi,u_i\rangle+\beta \Psi_i(u_i) \right\}.
\end{equation}

\end{itemize}

Observe that problem (\ref{model:problemsimplified}) is ``essentially" in the situation just described. Problem (\ref{compositeSP}) gives rise to two  convex optimization problems that are dual to each other:
\begin{equation*}\label{induced}
\begin{array}{rcl}
\Opt(P)&=&\min_{u_1\in U_1}\left[\overline{\Phi}(u_1):=\sup_{u_2\in U_2}\Phi(u_1,u_2)\right]\; (P)\\
\Opt(D)&=&\max_{u_2\in U_2}\left[\underline{\Phi}(u_2):=\inf_{u_1\in U_1}\Phi(u_1,u_2)\right]\;\;(D)\\
\end{array}
\end{equation*}
with $\Opt(P)=\Opt(D)$ if at least one of the sets $U_1$ and $U_2$ is bounded. 

Note that the distance generating functions define the Bregman distances $V_i(u_i,u_i')=\omega_i(u_i)-\omega_i(u_i')-\nabla\omega_i (u_i')^T(u_i-u_i')$ where $u_i', u_i\in U_i$ for $i=1,2.$  Given two scalars $\alpha_1>0, \alpha_2>0$,  we can build an aggregated distance generating function on $U=U_1\times U_2$ with 
$
\omega(u=[u_1;u_2])=\alpha_1\omega_1(u_1)+\alpha_2\omega_2(u_2),
$
which is compatible to the norm
$\|u=[u_1;u_2]\|^2=\alpha_1\|u_1\|_{(1)}^2+\alpha_2\|u_2\|_{(2)}^2.$

\paragraph{Composite Mirror Prox algorithm} We present in Algorithm \ref{CMP algorithm} the adaptation of composite Mirror Prox  introduced in~\cite{he2015mirror} for solving composite saddle point problem (\ref{compositeSP}). The  algorithm, generalizes the proximal gradient method with Bregman distances from the usual composite minimization to composite saddle  point problems and works ``as if" there are no non-smooth terms $\Phi_1$ and $\Phi_2$. 
\begin{algorithm}
\caption{Composite Mirror Prox  for Composite Saddle Point Problem}
\label{CMP algorithm}
{\bf Input:} $u_i^1\in U_i, \alpha_i>0, i=1,2$ and $\gamma_t>0$
\begin{algorithmic}
\FOR{$t=1,2,\ldots, T$}
  \STATE  $\;\;\;\hat u_i^t=\min\limits_{u_i\in U_i}\big\{\alpha_i V_i(u_i,u_i^t)+\langle \gamma_t\nabla_{i}\phi(u^t),u_i\rangle+\gamma_t \Psi_i(u_i) \big\}, i=1,2$
\STATE  $u_i^{t+1}=\min\limits_{u_i\in U_i}\big\{\alpha_i V_i(u_i,u_i^t)+\langle \gamma_t\nabla_{i}\phi(\hat u^t),u_i\rangle+\gamma_t \Psi_i(u_i) \big\}, i=1,2,$
\ENDFOR\\
\end{algorithmic}
%Output $u_{i,T} =\frac{\sum_{t=1}^T \gamma_t \hat u_i^t}{\sum_{t=1}^T\gamma_t},\; i =1, 2.$
\noindent{\bf Output:} $u_{1,T} =\frac{\sum_{t=1}^T \gamma_t \hat u_1^t}{\sum_{t=1}^T\gamma_t}$ and $u_{2,T} =\frac{\sum_{t=1}^T \gamma_t \hat u_2^t}{\sum_{t=1}^T\gamma_t}$

\end{algorithm}

For any set $U'\subset U$, let us define $\Theta[U']=\max_{u\in U'}V_\omega(u,u^1)$. We have from \cite{he2015mirror}

\begin{lemma}\label{lem:main}
Assume $\phi$  is $\cL$-Lipchitz w.r.t. the norm $\|\cdot\|$,  i.e. $\|\phi(u)-\phi(u')\|_*\leq \cL \|u-u'\|,$
where $\|\cdot\|_*$ is the dual norm.   
The solution $(u_{1,T}, u_{2,T})$ provided by the composite Mirror Prox algorithm with stepsize $0<\gamma_t\leq {\cL}^{-1}$, leads to the efficiency estimate
\begin{equation}\label{eq:accuracyforCSP}
\forall u=[u_1,u_2]\in U:\; \overline{\Phi}(u_{1,T}, u_2)-\underline{\Phi}(u_1,u_{2,T})\leq \frac{\Theta[U]}{\sum_{t=1}^T\gamma_t}.
\end{equation}
Moreover, if (P) is solvable with an optimal solution $u_1^*$ and set $\gamma_t={\cL}^{-1}$, then  one further has
\begin{equation}\label{eq:accuracyforCSP}
\overline{\Phi}(u_{1,T})-\Opt(P)\leq \frac{\Theta[u_1^*\times U_2]\cL}{T}.
\end{equation}
\end{lemma}

\paragraph{Remark.}  In the situation discussed above, the Mirror Prox algorithm achieves an optimal $O(1/t)$ convergence rate for solving composite saddle point problems. We emphasize that this is not the only algorithm available; alternative options include primal-dual algorithms~\cite{chambolle2011first, yang2015efficient}, hybrid proximal extragradient type algorithms~\cite{he2016accelerated,Tseng08}, just to list a few. Composite Mirror Prox differs from these algorithms in several aspects: i) it works for a broader class of problems beyond saddle point problem; ii) primal and dual variables are updated simultaneously which can easily accommodate parallelism; iii) it takes advantage of the geometry by utilizing Bregman distances; iv) the stepsize can be self-tuned using line-search without requiring a priori knowledge of Lipschitz constant.
Due to these differences, we particularly adopt Mirror Prox as our working horse to solve composite saddle point problems in the following.

%% file: sec3.tex
\section{Composite Mirror Prox for Penalized Poisson Regression}

\paragraph{Back to problem of interest.} Our ultimate goal is to address the saddle-point problem (\ref{model:problemsimplified}). Another key observation we have is that
\begin{lemma}\label{lem:simplefact}
Let $
y^+=\argmin_{y\in\bR^m_{++}}\left\{\frac{1}{2}\|y||_2^2+\langle \eta,y\rangle -\beta\sum_{i=1}^mc_i\log(y_i)\right\}$ given $\eta\in\bR^m$ and $\beta>0$,
then $y^+$ can be computed in closed-form as
\begin{equation}
\label{eq:poisson-fact}
y^+_i=Q^{\beta}(\eta_i):=\big(-\eta_i+\sqrt{\eta_i^2+4\beta c_i}\big)/2,\forall i=1,\ldots, m.
\end{equation}
\end{lemma}

In other words, the only non-Lipschitz term $p(y)=\sum_{i=1}^m c_i\log(y_i)$ in the objective is indeed proximal-friendly. This simple yet powerful fact has far-reaching implications, as we shall explain next.  We therefore propose to equip the domain $U=\{u=[x,y]: x\in \bR_+^n, y\in\bR_{++}^m\}$ with the mixed setup 
$\omega(u)=\alpha \omega_x(x)+\frac{1}{2}\|y\|_2^2$ with resepct to the norm $\|u\|=\sqrt{\alpha\|x\|_x^2+\|y\|_2^2} $
for some positive number $\alpha>0$.
 \noindent
\begin{minipage}{0.46\textwidth}
$\quad\,$ Recall the definition of proximal operator  and the fact~(\ref{eq:poisson-fact}).
The composite Mirror Prox algorithm for penalized Poisson regression simplifies to Algorithm~\ref{CMPforPoisson}.
In terms of iteration cost, Algorithm~\ref{CMPforPoisson} is embarrassingly efficient, given the closed-form solutions when updating both $x$ and $y$; for the latter case, it can even be done in parallel. When it comes to the iteration complexity, the algorithm achieves an overall $O(1/t)$ rate of convergence, which is significantly better than the usual $O(1/\sqrt{t})$ rate for non-smooth optimization~\cite{blair1985problem}. 
\end{minipage}
\begin{minipage}{0.02\textwidth}
\text{}
\end{minipage}
\begin{minipage}{0.5\textwidth}
\vspace{-0.4cm}
\begin{algorithm}[H]
\caption{\textbf{CMP} for Penalized Poisson Regression }
\label{CMPforPoisson}
{\bf Input:} $x^1\in \bR^n_+, y^1\in\bR^m_{++}, \alpha, \gamma_t>0$\\[-0.3cm]
\begin{algorithmic}
\FOR{$t=1,2,\ldots, T$}
  \STATE $\hat x^t=\Prox^{\gamma_t h/\alpha}_{x^t}\big(\gamma_t(s-A^Ty^t)/\alpha\big)$
  \STATE $ y^t_i= Q^{\gamma_t}(\gamma_t(a_i^Tx^t-y^t_i)), i=1,\ldots, m$
  \STATE $x^{t+1}=\Prox^{\gamma_t h/\alpha}_{x^t}\big(\gamma_t(s-A^T\hat y^t)/\alpha\big)$
  \STATE $y^{t+1}_i=Q^{\gamma_t}(\gamma_t(a_i^Tx^t-\hat y^t_i)), i=1,\ldots, m$
\ENDFOR\\
\end{algorithmic}
{\bf Output:} $x_{T} =\frac{\sum_{t=1}^T \gamma_t \hat x^t}{\sum_{t=1}^T\gamma_t}$
\end{algorithm}
\end{minipage}

\paragraph{Convergence analysis.} Denote $f(x)=L(x)+h(x)$. Given any subset $X\subset\bR^n_+$, let  $Y[X]:=\{y: y_i=1/(a_i^Tx), i=1,\ldots,m, x\in X\}$. Clearly, $Y[X]\subset\bR_{++}^m$. We arrive at 
\begin{theorem}\label{prop:Poisson}
Assume  we have some a priori information on the optimal solution to problem in (\ref{model:problemofinterest}): a convex compact set $X_0\subset \bR^n_{+}$ containing $x_*$ and a convex compact set $Y_0\subset\bR^m_{++}$ containing $Y[X_0]$. Denote $\Theta[X_0]=\max_{x\in X_0}V_\omega(x,x^1)$ and $\Theta[Y_0]=\max_{y\in Y_0}\frac{1}{2}\|y-y^1\|_2^2$.
Let 
$$\cL= \|A\|_{x\to 2}:=\max_{x\in\bR^n_+:\|x\|_x\leq 1} \{\|Ax\|_2\}$$ and let stepsizes in Algorithm~2 satisfy  $0<\gamma_t\leq \sqrt{\alpha} \cL^{-1}$ for all $t>0$. We have
\begin{equation}
f(x_{T})-f(x_*)\leq\frac{\alpha\Theta[X_0]+\Theta[Y_0]}{{\sum}_{t=1}^T\gamma_t}
\end{equation}
 In particular, by setting $\gamma_t=\sqrt{\alpha} \cL^{-1}$ for all $t$ and $\alpha=\Theta[Y_0]/\Theta[X_0]$, one further has 
\begin{equation}\label{eq:Poissonbound}
f(x_{T})-f(x_*)\leq \frac{\sqrt{\Theta[X_0]\Theta[Y_0]}\|A\|_{x\to 2}}{T}.
\end{equation}

\end{theorem}

\paragraph{Remark 1.} Note that Algorithm~\ref{CMPforPoisson} works without requiring $a_i^Tx>0,\forall i$, or any global Lipschitz continuity of the original objective function. The sets $X_0$ and $Y_0$ only appear in the theoretical guarantee yet are never involved in the computations along the iterations of  Algorithm~\ref{CMPforPoisson}. 

Furthermore, one can easily get candidate sets $X_0$ and $Y_0$\footnote{Knowing the geometry of such set could also help us determine favorable proximal setups, which we will illustrate in Sec.~\ref{sec:application}.} In principle, we can at least say that
\begin{equation}
\begin{array}{c}
X_0=\{x\in\bR_+^n: s^Tx+h(x)\leq \sum_{i=1}^m c_i\}.
\end{array}
\end{equation}
Clearly, $X_0$ is convex and compact. The reason why $x_*\in X_0$ is due to the following fact.  See Appendix~\ref{sec:appendix_PoissonReg} for proof. 
\begin{proposition}\label{prop:domain}
The optimal solution $x_*$ to the problem in (\ref{model:problemofinterest}) satisfies
$
 s^Tx_*+h(x_*)= \sum_{i=1}^m c_i.
$
\end{proposition}

\paragraph{Remark 2.} Theorem~\ref{prop:domain} implies that the performance of Algorithm~\ref{CMPforPoisson} is essentially determined by the distance between the initial solution $(x^1,y^1)$ to the optimal solution $(x_*,y_*)$. Therefore, if the initial solution is close enough to the optima, then one can expect the algorithm to converge quickly. In practice, the optimal choice of $\alpha =\frac{1}{2}\|y^1-y_*\|_2^2/V_\omega(x_*,x^1)$ is often unknown. One can instead select $\alpha$ from empirical considerations, for instance by treating $\alpha$ as a hyper-parameter and tuning it using cross-validation. 

%% file: sec4.tex
\section{Randomized Block Mirror Prox for Large-Scale Applications} \label{sec:extension}
While the saddle point reformulation eliminates the non-Lipschitz continuity suffered by the original problem, it also requires the introduction of $m$ dual variables, where $m$ equals the number of datapoints. Hence, Algorithm~2 is mostly appropriate for applications with reasonably large samples. Tackling extremely large-sample datasets requires additional computation and memory cost.  

We propose a randomized block-decomposition variant of composite Mirror Prox, that is appropriate for large-sample datasets. Block-coordinate optimization has received much attention and success recently for solving high-dimensional convex minimization problems; see~\cite{ nesterov2012efficiency, Shalev-Shwartz:2013, lu2013complexity, richtarik2014iteration, dang2015sbmd} and reference therein. However, to the best of our knowledge, 
this is the first time that a randomized block-coordinate variant of Mirror Prox is developed to solve saddle point problems and a more general class of variational inequalities.  

\noindent
\begin{minipage}{0.45\textwidth}
$\quad\,$ For the sake of simplicity, here we only present the algorithm customized specifically for the Poisson likelihood models of our interest and leave the general results on variational inequalities in appendix for interested readers. In Appendix~\ref{sec:appendix_RBMP} and \ref{sec:appendix_partRBMP}, we introduce CMP algorithm with  fully randomized and partially randomized updating rules, respectively, and provide detailed convergence analysis. We emphasize that the randomized block Mirror Prox algorithm shares some similarity with few existing works based on primal-dual schemes, e.g. the SPDC algorithm \cite{zhang2015stochastic} and the RPD algorithm\cite{dang2014randomized} when applying to saddle point problems, but they are algorithmically different, as previously discussed in Sec.~\ref{sec:preliminary}. \\
\end{minipage}
\begin{minipage}{0.02\textwidth}
\text{}
\end{minipage}
\begin{minipage}{0.48\textwidth}
\vspace{-0.55cm}
\begin{algorithm}[H]
\caption{Randomized Block Mirror Prox\\ ({\bf RB-CMP}) for Penalized Poisson Regression }
\label{alg:randomizedCMP}
{\bf Input:} $x^1\in \bR^n_+, y^1\in\bR^n_{++}, \gamma_t>0$\\[-0.3cm]
\begin{algorithmic}
\FOR{$t=1,2,\ldots, T$}
  \STATE Randomly pick a block $k_t\in\{1,2,\ldots,b\}$
  \STATE $\begin{array}{rl}
&\hat x^t=\Prox^{\gamma_t h}_{x^t}\big(\gamma_t(s-A^Ty^t)\big)\\
\vspace{2mm}
&\hat y^t_k=\left\{ \begin{array}{ll}
Q^{\gamma_t}(\gamma_t(A_kx^t-y^t_k)),&\;\; k=k_t\\
x^t_k, &\;\; k\neq k_t
\end{array}\right.
\end{array}$
 \STATE $\begin{array}{rl}
&x^{t+1}=\Prox^{\gamma_t h}_{x^t}\big(\gamma_t(s-A^T\hat y^t)\big)\\
\vspace{2mm}
&y^{t+1}_k=\left\{ \begin{array}{ll}
Q^{\gamma_t}(\gamma_t(A_kx^t-y^t_k)),& k=k_t\\
x^t_k, & k\neq k_t
\end{array}\right.
\end{array}$
\ENDFOR\\
\end{algorithmic}
{\bf Output:} $x_{T} =\frac{\sum_{t=1}^T \gamma_t \hat x^t}{\sum_{t=1}^T\gamma_t}$, $y_{T} =\frac{\sum_{t=1}^T \gamma_t \hat y^t}{\sum_{t=1}^T\gamma_t}$
\end{algorithm}
\end{minipage}

With a slight abuse of notation, let us denote $y=[y_1;\ldots;y_b]$ and $A=[A_1;\ldots;  A_b]$, where $y_k\in\bR^{m_k}, A_k\in\bR^{m_k\times n}, k=1,\ldots,b$  such that $m_1+m_2+\ldots+m_b=m$.  The randomized block Mirror Prox algorithm tailored to solve (\ref{model:problemsimplified})  is described  in Algorithm \ref{alg:randomizedCMP}.

\begin{theorem}\label{prop:RandomizedCMP}
Let $\cL_k:=\max_{x\in\bR^n_+:\|x\|_x\leq 1} \{\|A_kx\|_2\}$ and stepsizes in Algorithm \ref{alg:randomizedCMP} satisfy  $0<\gamma_t\leq 1/\cL_M:=\min_{k=1,\ldots, b}\{(\sqrt{2b} \cL_k)^{-1}\}$ for all $t>0$. Denote $\Phi(x,y)$ as the saddle function in (\ref{model:problemsimplified}), then for any $x\in X_0, y\in Y_0$. Under same assumptions as in Theorem~\ref{prop:Poisson}, we have
\begin{equation}
\bE[\Phi(x_{T}, y)-\Phi(x,y_{T})]\leq\frac{\Theta[X_0]+b\cdot \Theta[Y_0]}{{\sum}_{t=1}^T\gamma_t}.
\end{equation}
In particular, when setting $\gamma_t\equiv 1/\cL_M, \forall t>0$, we have for any $x\in X_0, y\in Y_0$,
\begin{equation}
\bE[\Phi(x_{T}, y)-\Phi(x,y_{T})]\leq\frac{(\Theta[X_0]+b\cdot \Theta[Y_0])\cL_M}{T}.
\end{equation}
\end{theorem}
\paragraph{Remark 3.} A full description of the algorithm and convergence analysis is presented in Appendix~\ref{sec:appendix_partRBMP}. Similar to its full batch version, the randomized block Mirror Prox algorithm also enjoys the $O(1/t)$ convergence rate, but with relatively cheaper iteration cost. 
The above error bound does not necessarily imply $\bE[f(x_T)-f_*]\leq O(1/T)$, as one would wish to have. Indeed, establishing such a result is notoriously difficult in general when considering randomized algorithms for general saddle point problems, as emphasized in~\cite{dang2014randomized}. We shall therefore leave this for future investigation. 

%% file: sec5.tex
\section{Experiments: Learning and Inferences on Diffusion Networks}\label{sec:application}
In this section, we present illustrations of the proposed approaches as applied to two emerging applications in estimating point processes introduced in the beginning. Due to the space limit, we put our experimental results on Poisson imaging in the Appendix~\ref{sec:appendix_PET}. Detailed algorithms tailored for each model and  experimental setups can also be found in Appendix~\ref{sec:appendix_network}.   
 
\vspace{-0.3cm}
\subsection{Social network estimation} 
Given a sequence of events $\{(u_j,t_j)\}_{j=1}^m$, the goal is to estimate the influence matrix among users. We focus on the convex formulation as posed in \cite{zhou2013learning} 
\begin{equation}\label{model:plainHawkes}
\begin{array}{c}
\min_{x\geq 0, X\geq 0} L(x,X) +\lambda_1\|X\|_1+\lambda_2\|X\|_\nuc
\end{array}
\end{equation}
$L(x,X):= \sum_{u=1}^U [Tx_u+\sum_{j=1}^m X_{uu_j}G(T-t_j)]
-\sum_{j=1}^m\log\big(x_{u_j}+\sum_{k:t_k<t_j}X_{{u_j}{u_k}}g(t_j-t_k)\big)$ is the negative log-likelihood term, $\|\cdot\|_\nuc$ is  nuclear norm.  We first focus on the case $\lambda_2=0$, since our purpose here is to investigate the effect of non-Lipschitz continuity. 

\paragraph{Experimental setup.} We compare the proposed composite Mirror Prox (CMP) and its randomized block variant (RB-CMP) to Mirror Descent (MD)for compositive objective ~\cite{duchi2010composite} both on synthetic and real Twitter datasets.  The synthetic dataset consists of 50 users and 50,000 events. The Twitter dataset consists of 100 users and 98,927 events. Based on Proposition~\ref{prop:domain}, we can see that the optimal solution is contained in some bounded simplex. Hence, for all algorithms, we use  entropy distance generating function,  i.e. $\omega_x(x)=\sum_{u} x_u\log(x_u)$ and $\omega_X(X)=\sum_{u,u'}X_{uu'}\log(X_{uu'})$.  Therefore, our algorithms give multiplicative updates for $x$ and $X$. 

\paragraph{Numerical results.}  We run the three algorithms with their best-tuned parameters \footnote{For MD and RB-CMP, we tune the stepsize using cross-validation; for CMP, the stepsize is self-tuned via line-search.},  respectively, under different regularization parameters $\lambda_1\in\{0.01, 1, 100\}$. We evaluate their relative sub-optimality $\frac{f(x_t)-f_*}{f(x_1)-f_*}$ vs number of effective passes through data, where $f$ is the overall objective and $f_*$ is estimated by running the best algorithm long enough. The results are presented in Figure~\ref{fig:network}, which demonstrate that our composite Mirror Prox algorithm performs significantly and consistently better than Mirror Descent, and the randomized block variant further improves the performance especially on  large real-world datasets.

\begin{figure*}[!ht]
\begin{center}
  \begin{subfigure}[t]{.3\textwidth}
    \includegraphics[width=\textwidth]{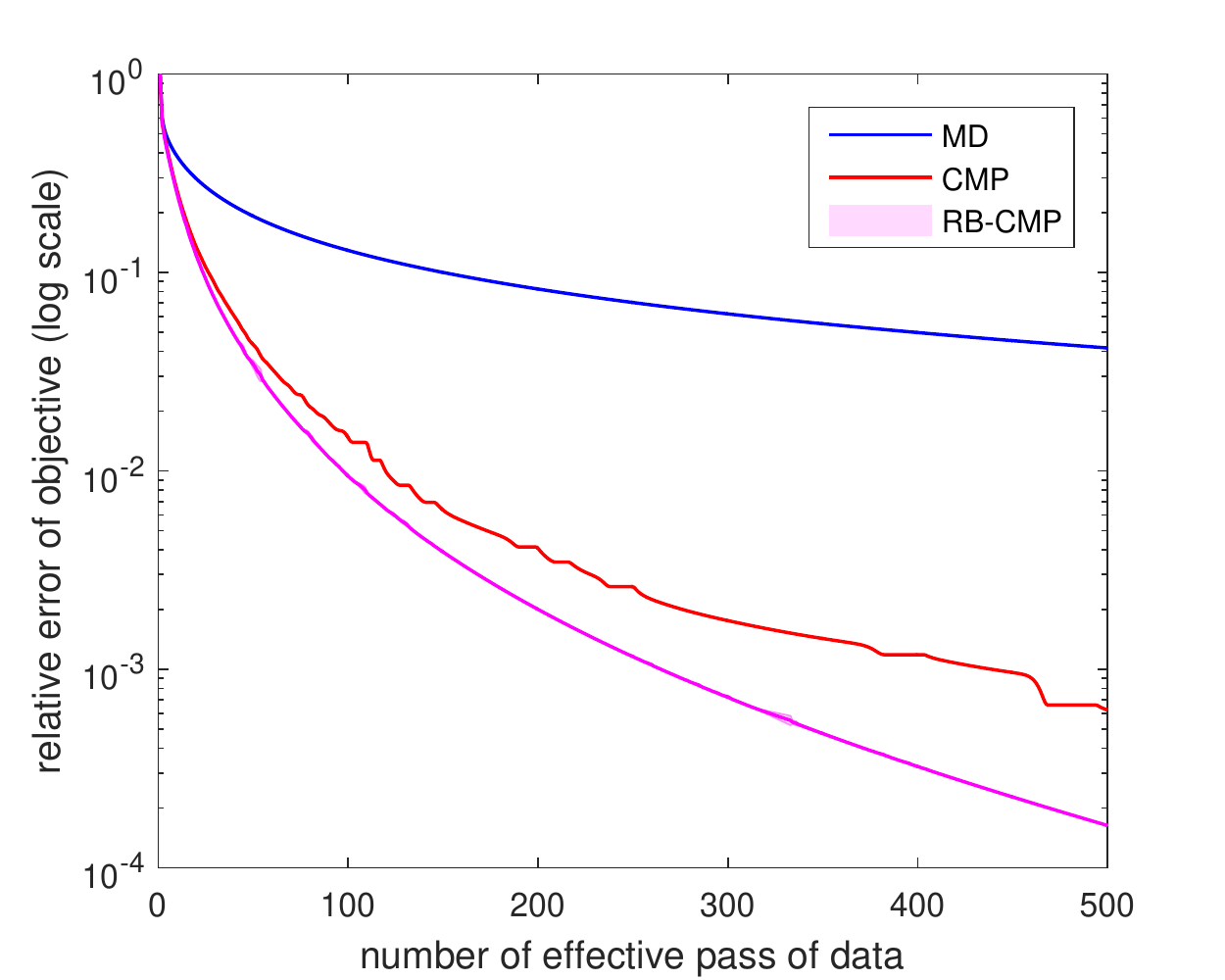}
    \caption{synthetic dataset, $\lambda_1=0.01$}
  \end{subfigure}%
  \begin{subfigure}[t]{.3\textwidth}
    \includegraphics[width=\textwidth]{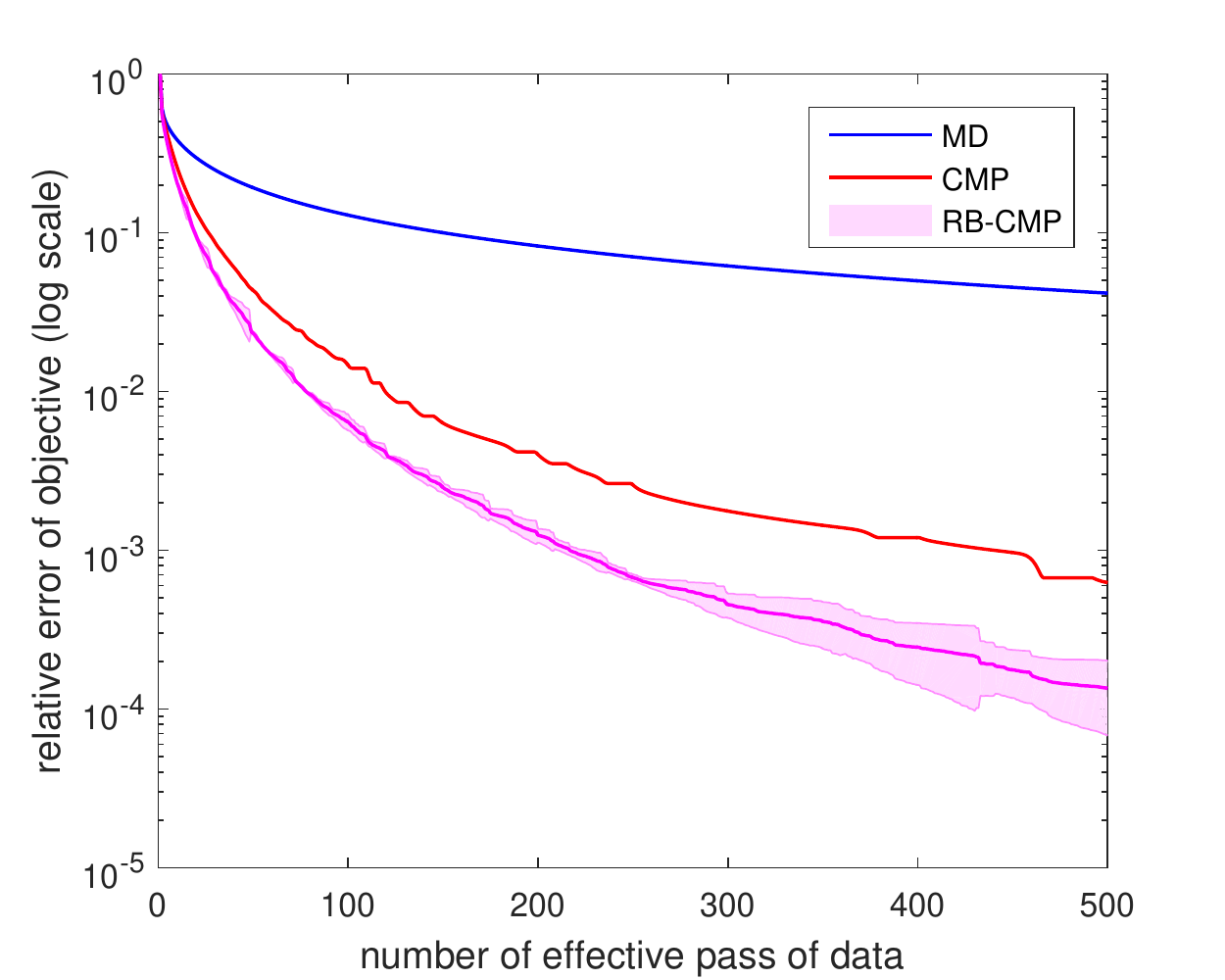}
    \caption{synthetic dataset, $\lambda_1=1$}
  \end{subfigure}%
  \begin{subfigure}[t]{.3\textwidth}
    \includegraphics[width=\textwidth]{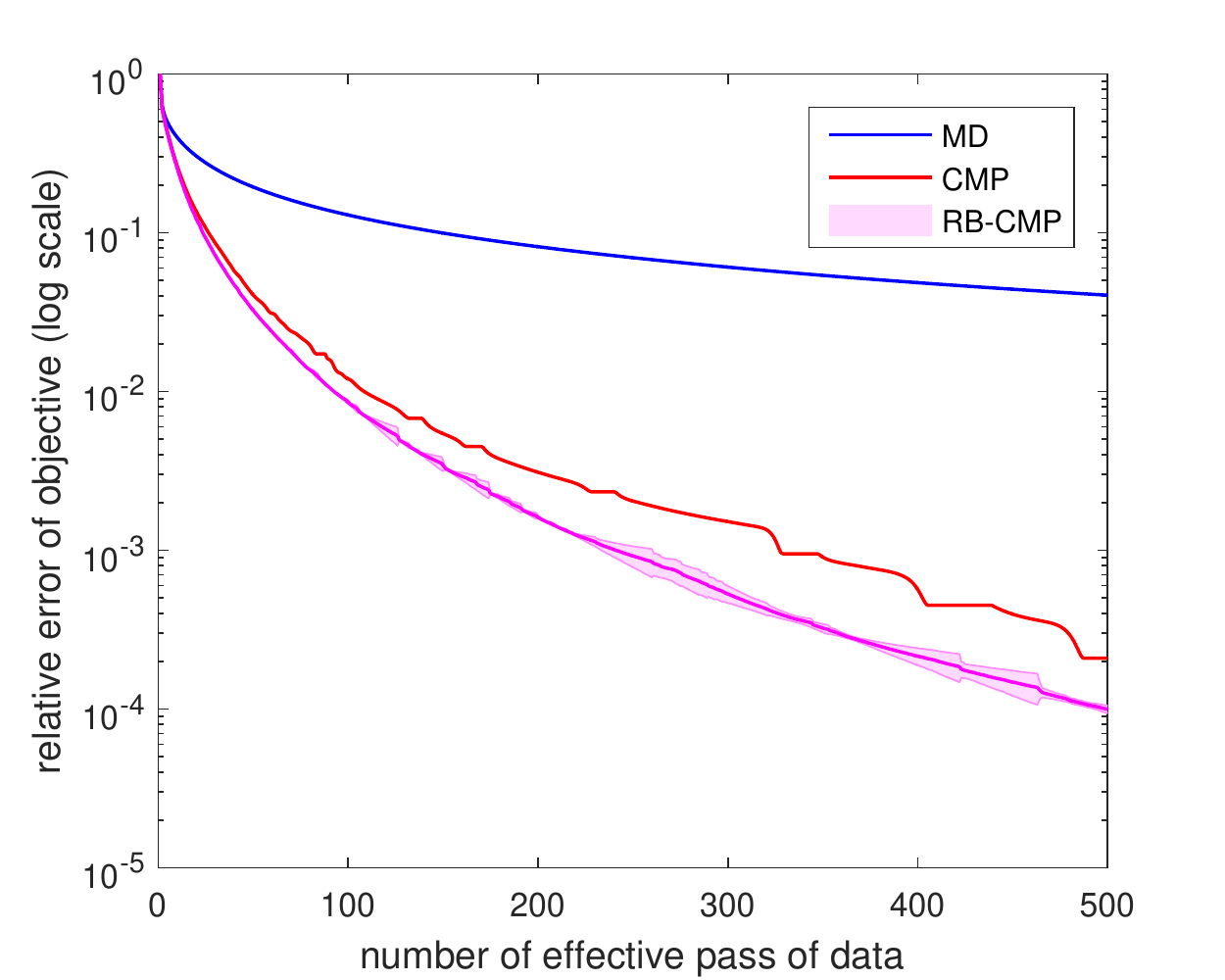}
     \caption{synthetic dataset, $\lambda_1=100$}
  \end{subfigure}%
  
    \begin{subfigure}[t]{.3\textwidth}
    \includegraphics[width=\textwidth]{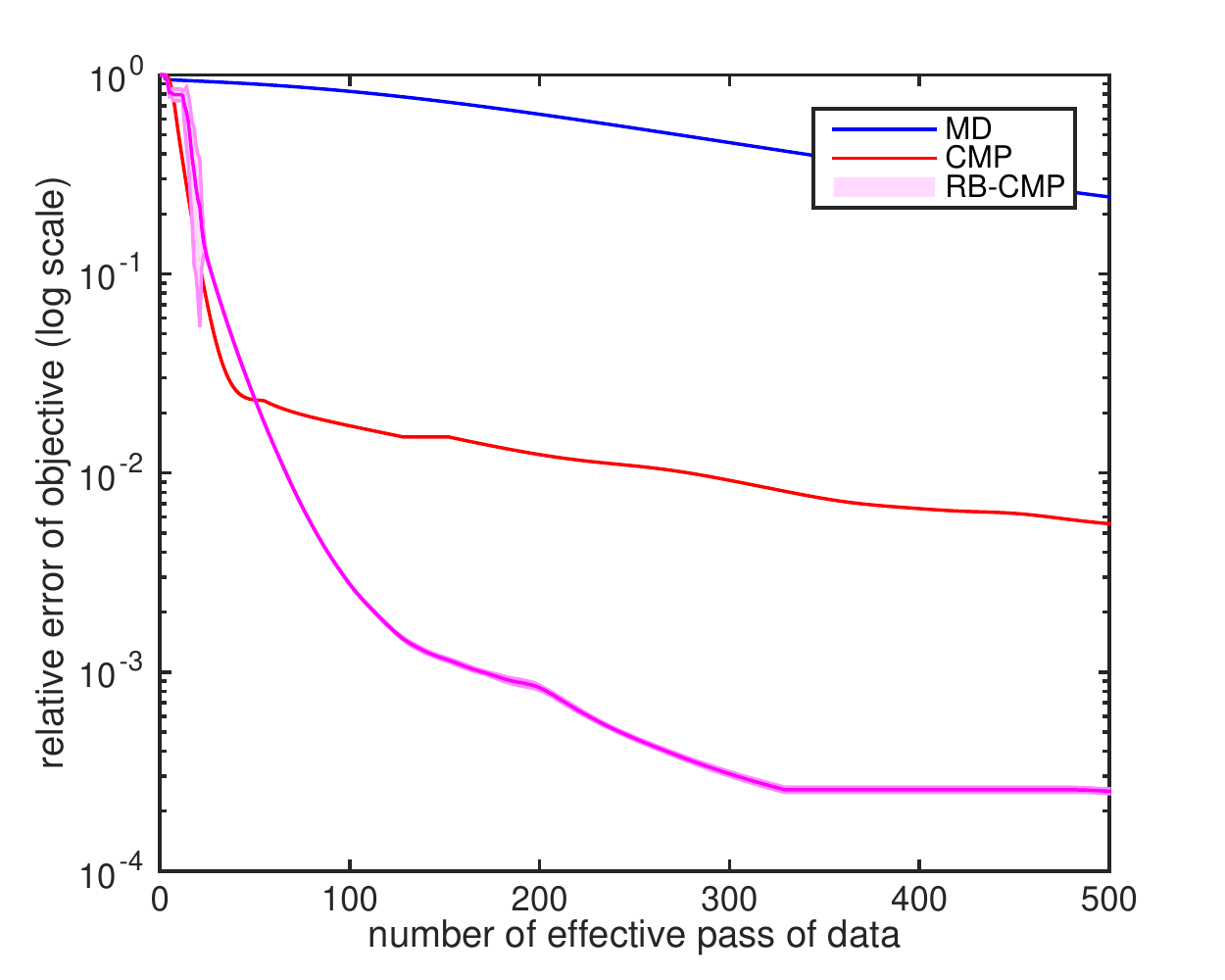}
    \caption{Twitter dataset, $\lambda_1=0.01$}
  \end{subfigure}%
  \begin{subfigure}[t]{.3\textwidth}
  \includegraphics[width=\textwidth]{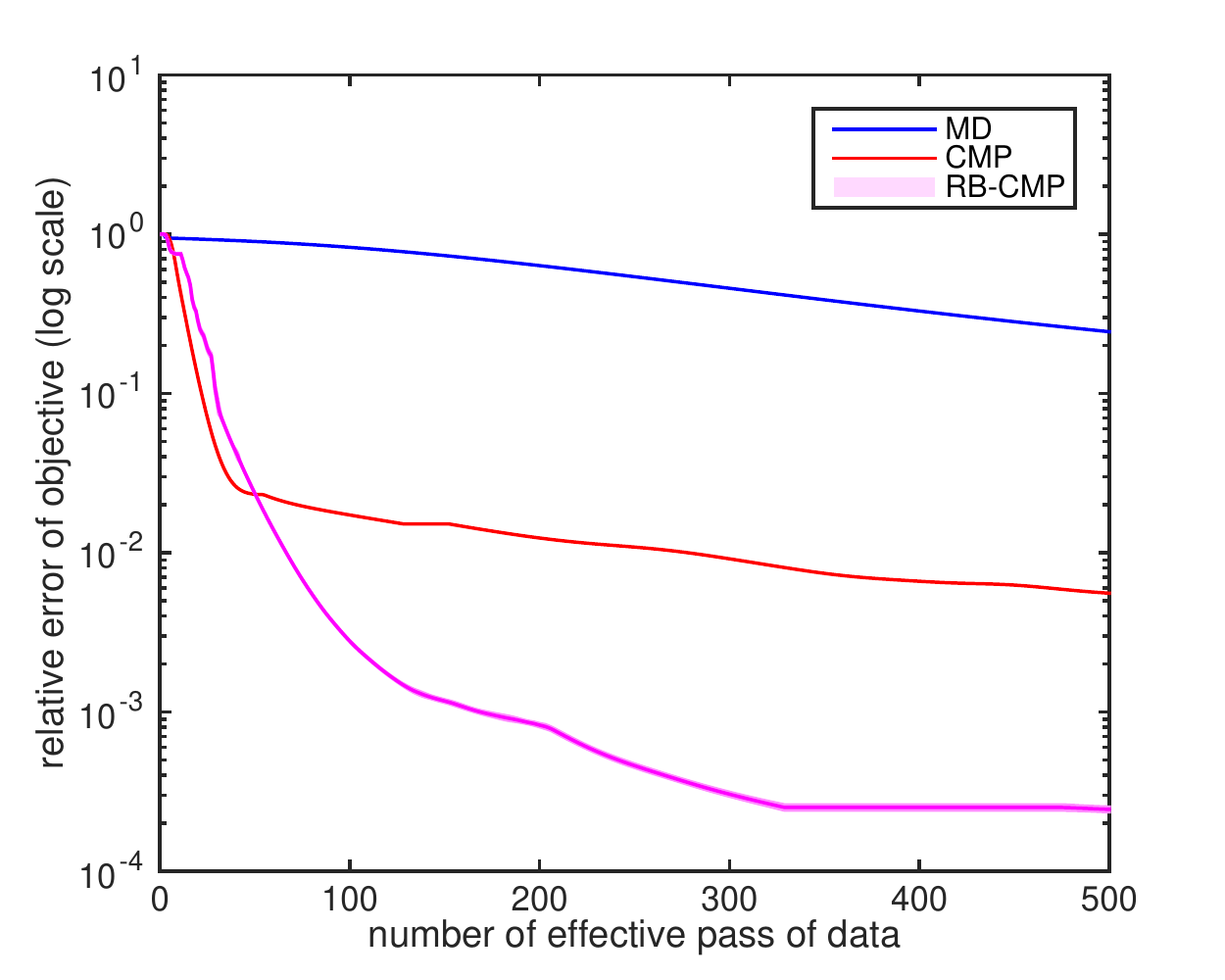}
    \caption{Twitter  dataset, $\lambda_1=1$}
  \end{subfigure}%
  \begin{subfigure}[t]{.3\textwidth}
    \includegraphics[width=\textwidth]{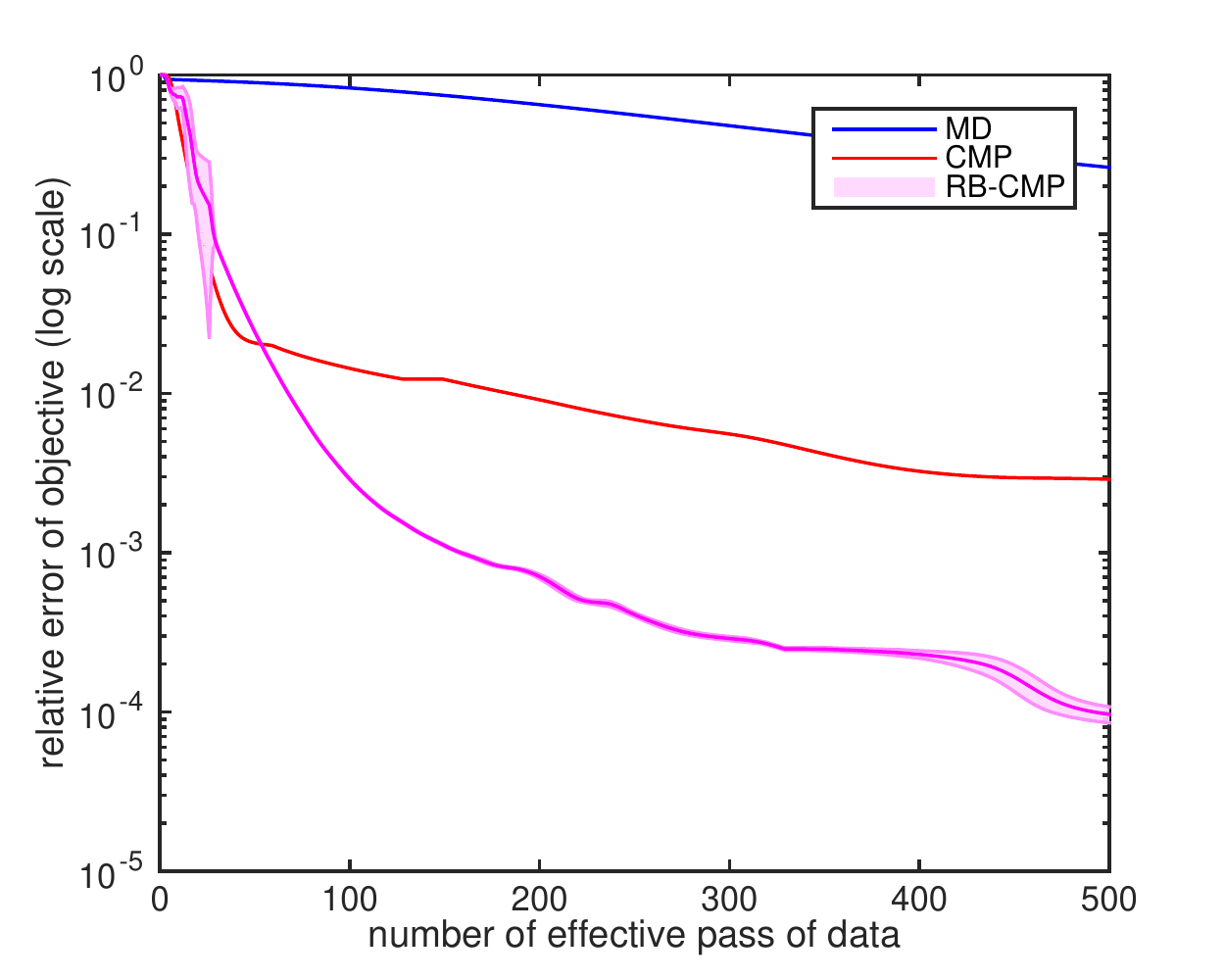}
     \caption{Twitter  dataset, $\lambda_1=100$}
  \end{subfigure}%
  \caption{Social network estimation \\[-0.5cm]}
  \label{fig:network}
  \end{center}
\end{figure*}

\subsection{Temporal recommendation system} 
Given a sequence of events $\{\cT^{u,i}\}_{u,i}$ for each user-item pair $(u,i)$, we consider the alternative convex reformulation posed in \cite{Nan15} for modeling the temporal behaviors of user-item preference,
\begin{equation}\label{model:penalized2}
\begin{array}{c}
\min\limits_{X_1\geq0, X_2\geq 0, Y_1,Y_2} \;L(X_1,X_2)+\lambda_1\|Y_1\|_\nuc+\lambda_2\|Y_2\|_\nuc+\rho \|X_1-Y_1\|^2_2+\rho \|X_2-Y_2\|^2_2
\end{array}
\end{equation}
where the negative log-likelihood term is  $L(X_1,X_2)= \sum_{u,i}\big[TX_1^{u,i}+\sum_{t_j\in \cT^{u,i} }[X_2^{u,i}G(T-t_j)-\log(X_1^{u,i}+X_2^{u,i}\sum_{t_k<t_j}g(t_j-t_k))]\big].$  Matrices $X_1$ and $X_2$ denote the base intensity and self-exciting coefficients for all user-item pair, variables $Y_1$ and $Y_2$ are copies of $X_1$ and $X_2$.\\
 
 \noindent
\begin{minipage}{0.4\textwidth}
\paragraph{Experimental setup.} We compare CMP to serveral algorithms including Mirror Descent (MD) \cite{duchi2010composite}, proximal gradient descent (PG)  and  accelerated proximal gradient (APG)~\cite{NesCompMin, Nan15}. The stepsizes of PG and APG are selected adaptively since the objective is non-globally Lipschitz continuous. 
\end{minipage}
\begin{minipage}{0.02\textwidth}
\text{}
\end{minipage}
\begin{minipage}{0.5\textwidth}
{\small
\begin{table}[H]
\vspace{-0.5cm}
\begin{center}
\caption{Datasets for temporal recommendation}
\label{table}
\begin{tabular}{|c|c|c|c|c|}
\hline
dataset &user & item  & pair & event\\
\hline\hline
synthetic &64 &  64 & 2048 & 2048000\\
\hline
last.fm (small) &297&  423 & 492 & 31353\\
\hline
last.fm (medium) &568&1162 & 1822 &127724\\
\hline
last.fm (large)&727&2247 & 6737 &454375\\
\hline
\end{tabular}
\end{center}
\end{table}
}
\end{minipage}

\paragraph{Numerical results.} We run the experiments on both synthetic can real-world datasets as described in Table~\ref{table}. The number of events in the \textit{last.fm} dataset ranges from 30,000 to 500,000. We set the regularization parameters to be the same and range from $\{0.1,1,10\}$. The results are presented in Figure~\ref{fig:RecSys}.  Figure~\ref{fig:RecSys} clearly indicates that i) using non-Euclidean setup significantly improves the performance ii) our algorithm performs considerably better than  Mirror Descent and even accelerated proximal gradient method in practice. \\[-0.5cm]

\begin{figure*}[!ht]
\begin{center}
  \begin{subfigure}[t]{.3\textwidth}
    \includegraphics[width=\textwidth]{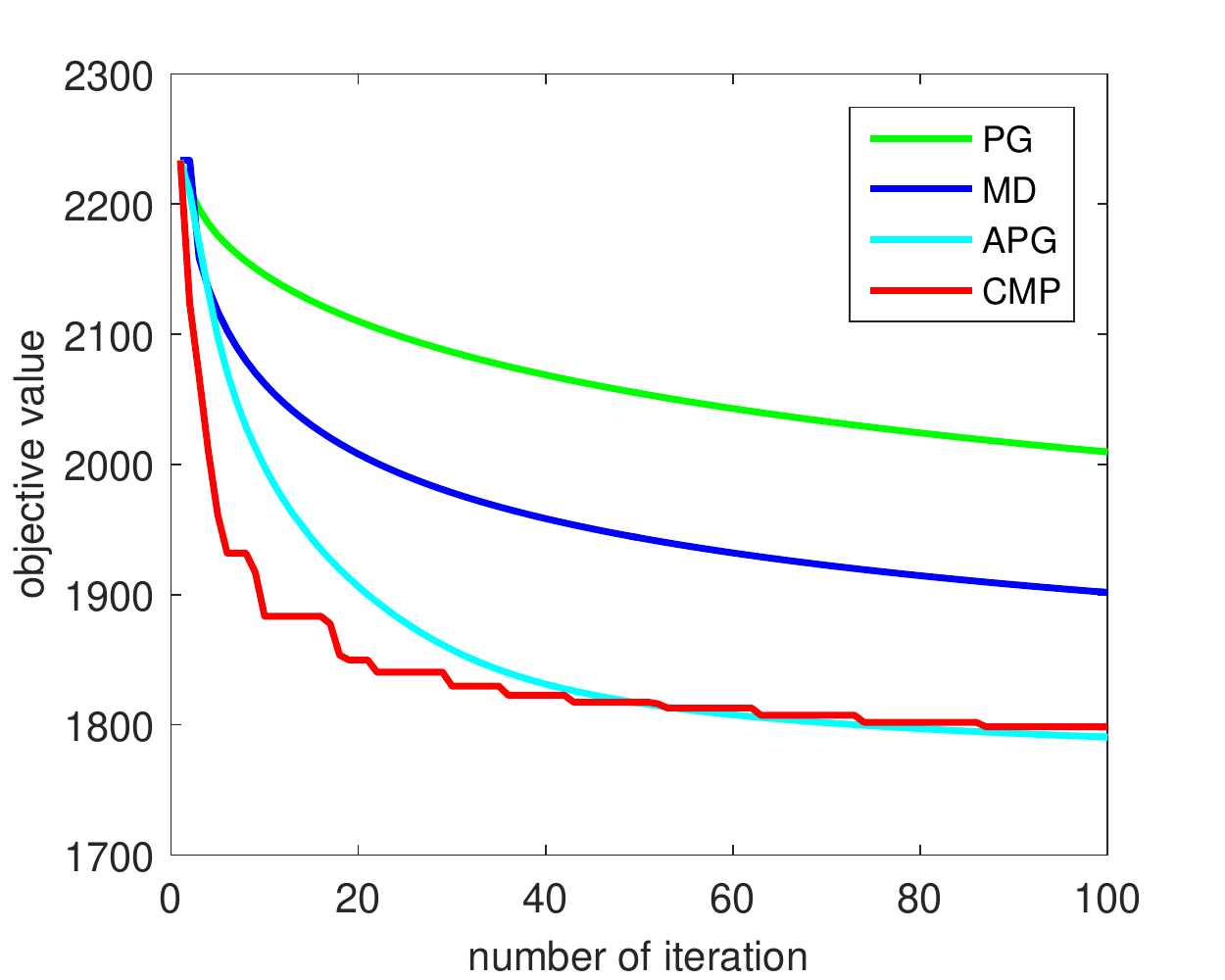}
    \caption{synthetic dataset, $\lambda=0.1$}
  \end{subfigure}%
  \begin{subfigure}[t]{.3\textwidth}
    \includegraphics[width=\textwidth]{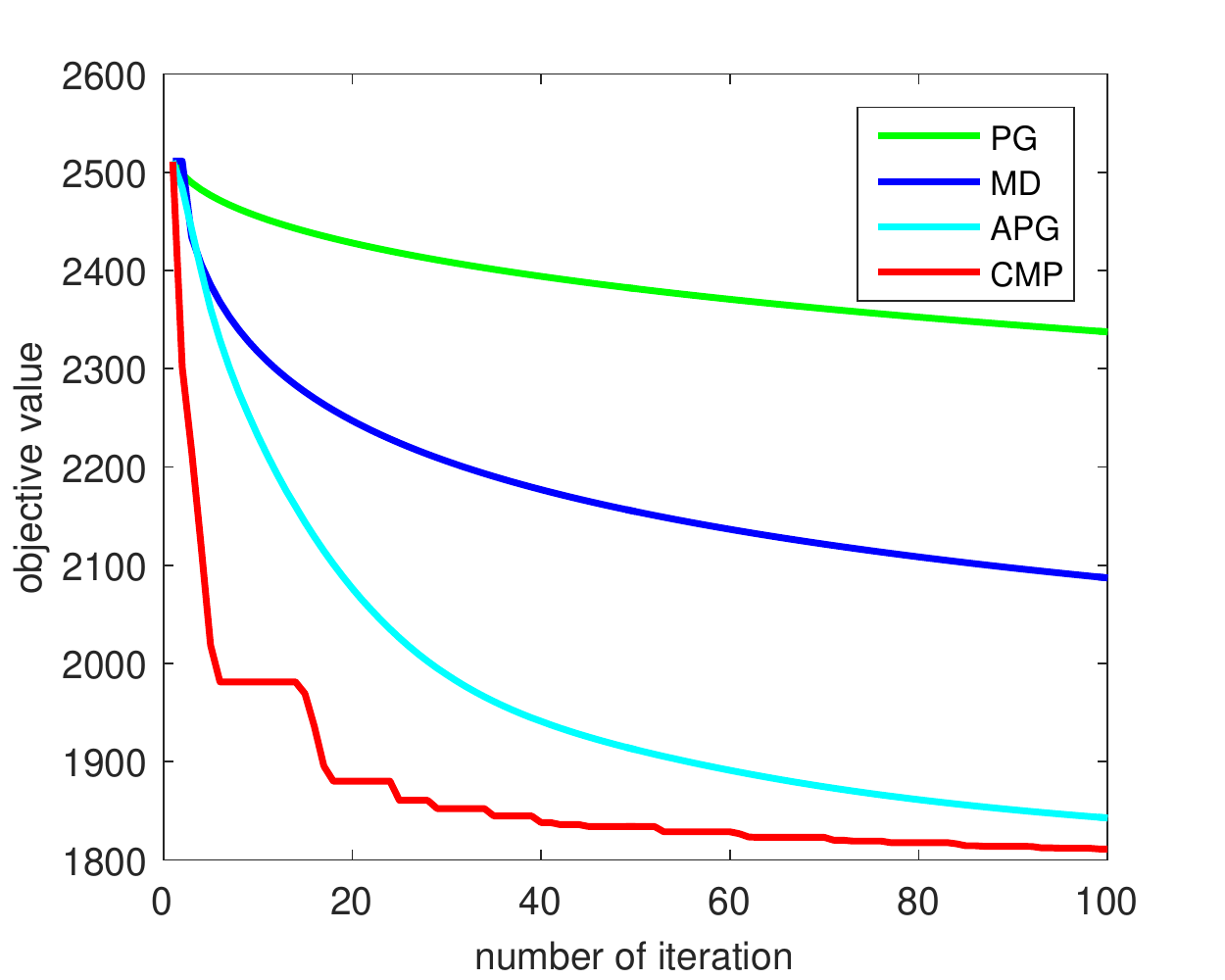}
    \caption{synthetic dataset, $\lambda=1$}
  \end{subfigure}%
  \begin{subfigure}[t]{.3\textwidth}
    \includegraphics[width=\textwidth]{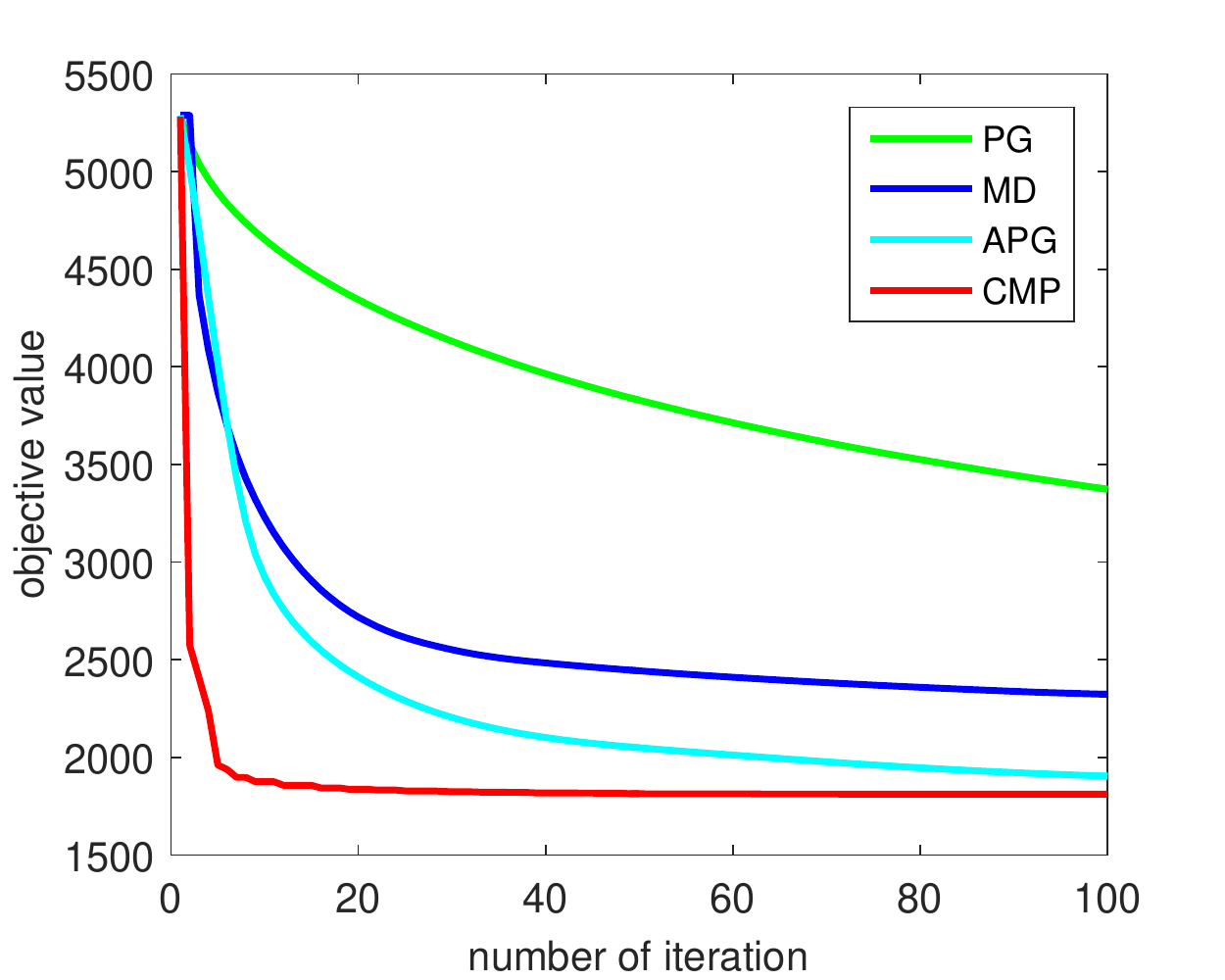}
     \caption{Synthetic dataset, $\lambda=10$}
  \end{subfigure}%
   
    \begin{subfigure}[t]{.3\textwidth}
    \includegraphics[width=\textwidth]{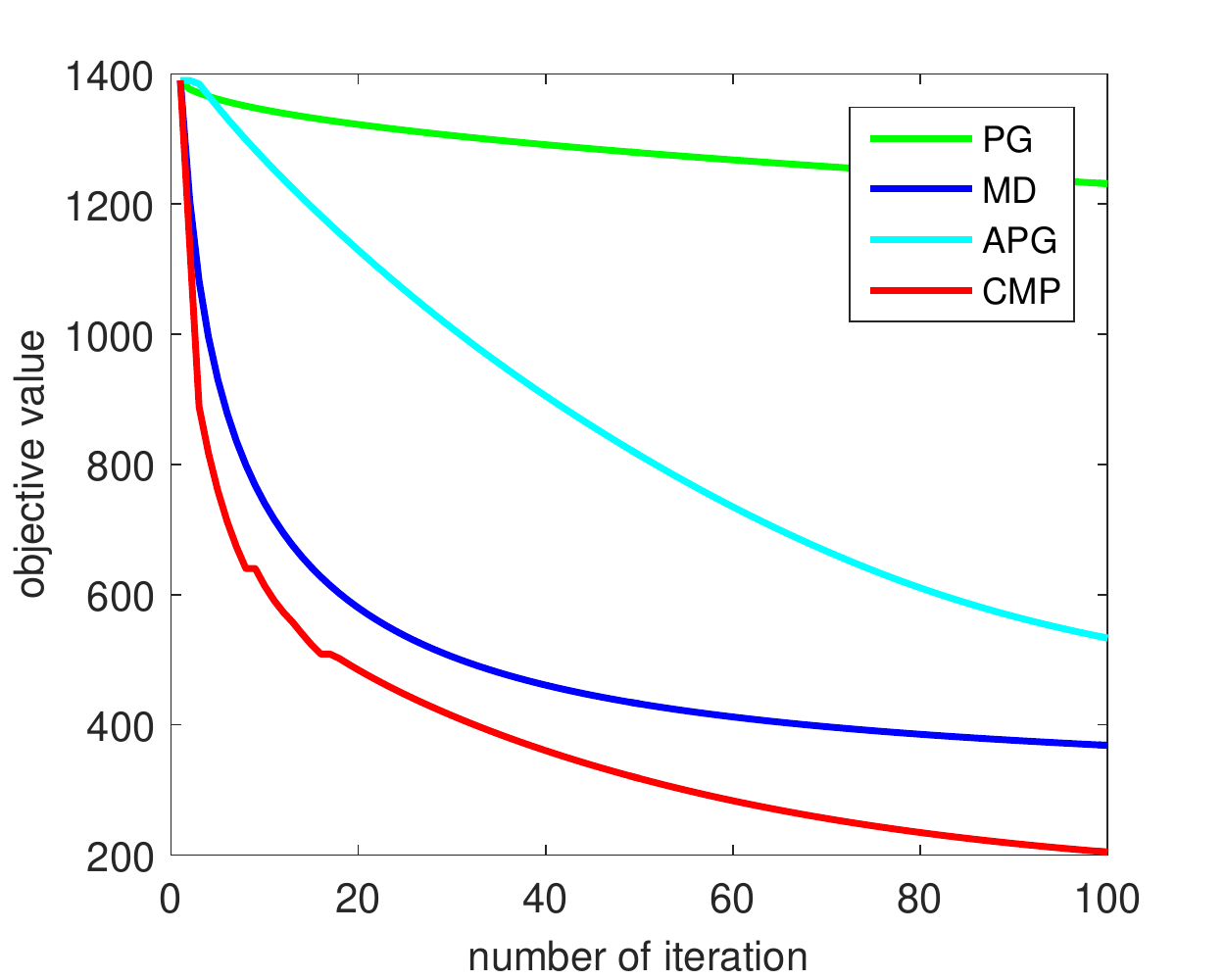}
    \caption{\textit{last.fm} dataset (small)}
  \end{subfigure}%
  \begin{subfigure}[t]{.3\textwidth}
    \includegraphics[width=\textwidth]{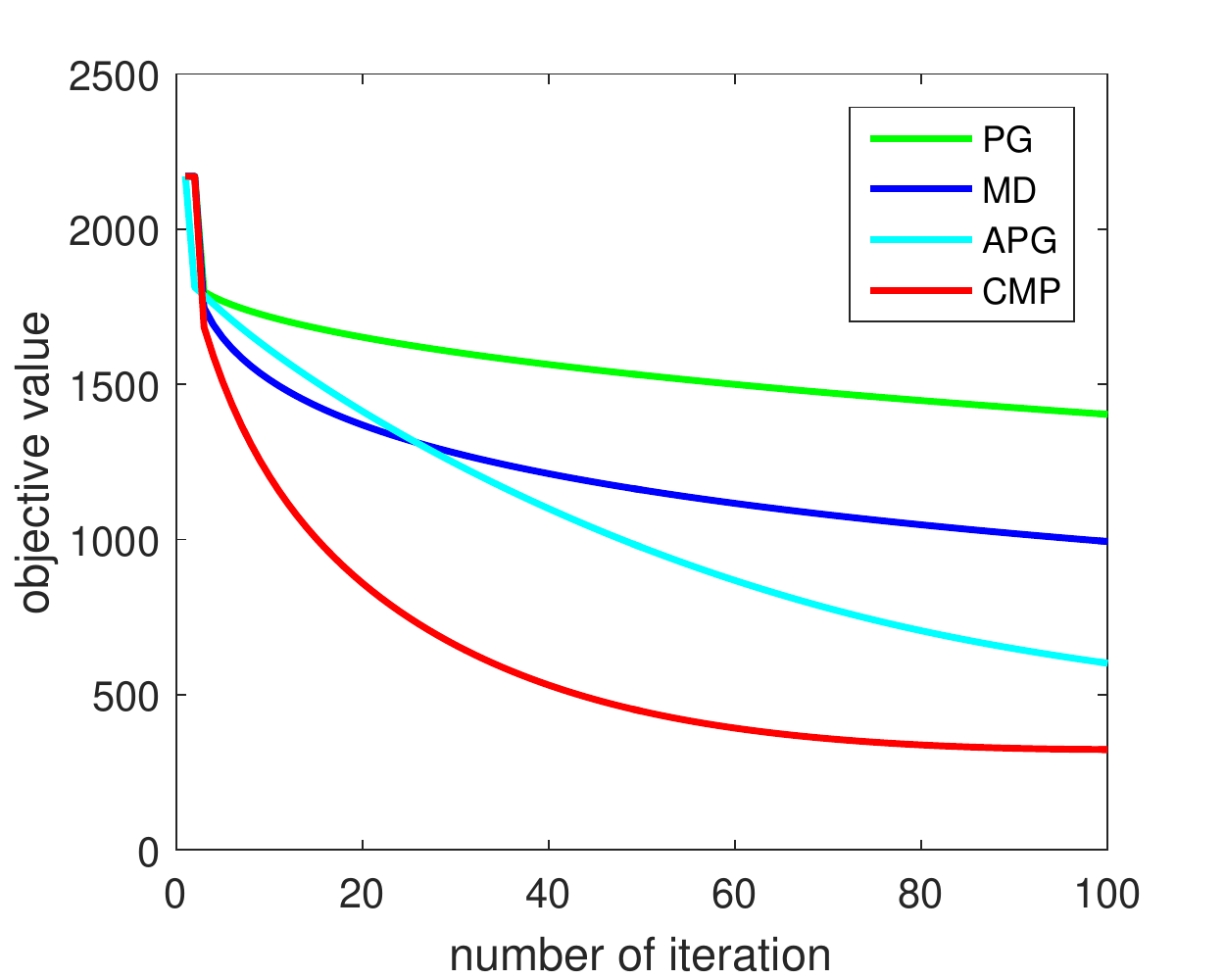}
    \caption{\textit{last.fm} dataset (medium)}
  \end{subfigure}%
  \begin{subfigure}[t]{.3\textwidth}
    \includegraphics[width=\textwidth]{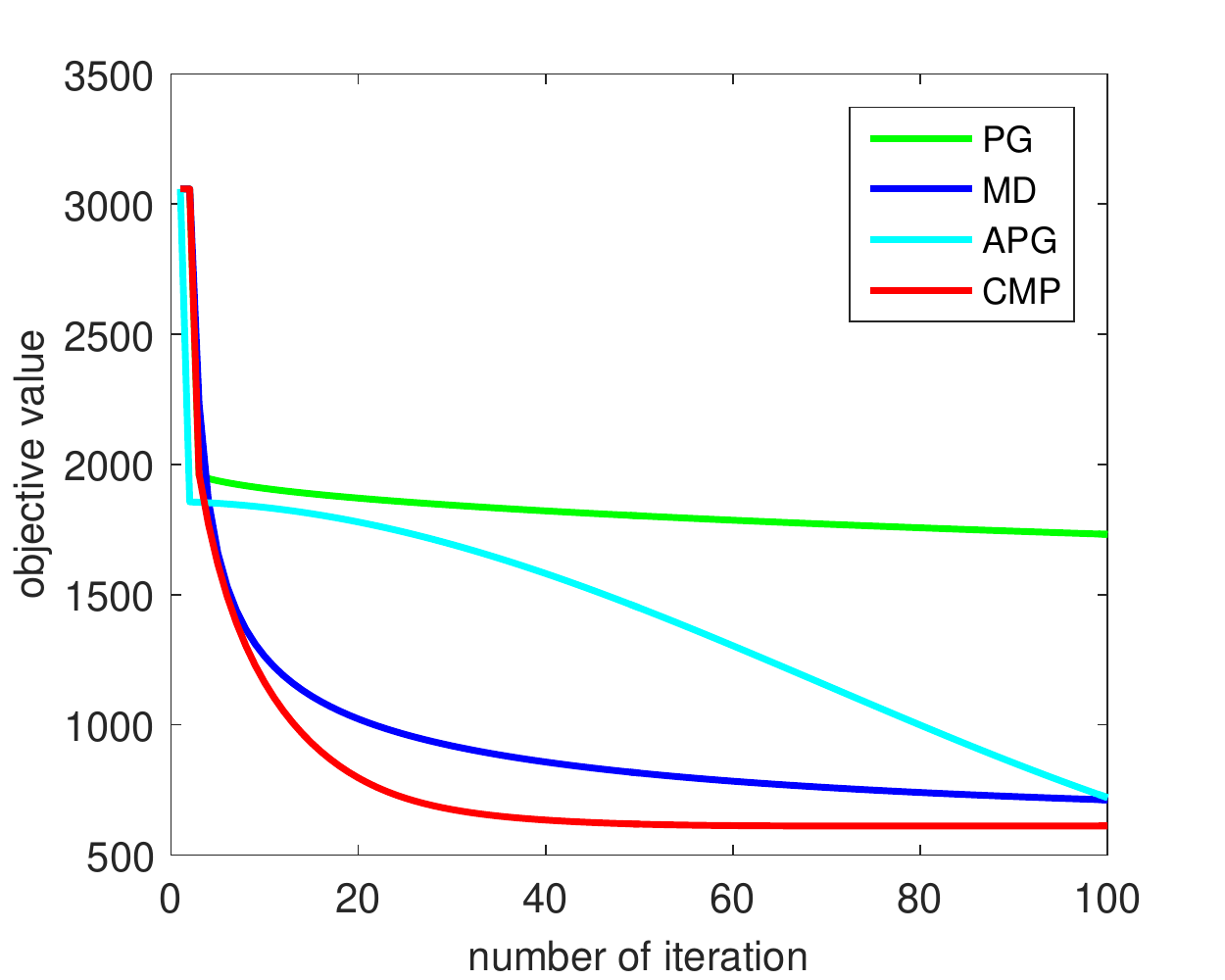}
     \caption{\textit{last.fm}  dataset (large)}
  \end{subfigure}%
  \caption{Temporal Recommendation System\\[-0.5cm]}
  \label{fig:RecSys}
  \end{center}
\end{figure*}

%% file: sec6.tex
\section{Conclusion}
We presented a new family of algorithms that speed up a wide range of point process applications in machine learning.  The proposed algorithms resolve the long-standing issue with non-Lipschitzness of Poisson likelihood models and enjoy a $O(1/t)$ convergence rate, in contrast to the typical $O(1/\sqrt{t})$ rate for non-smooth optimization. Both on synthetic and real-world data, the proposed algorithms outperforms Mirror Descent and Accelerated Proximal Gradient. 
For future work, since the proposed algorithm performs especially well in the first iterations, it would be interesting to investigate how to optimize up to statistical accuracy. 

%% file: appendix.tex
\begin{appendix}
\thispagestyle{plain}
\begin{center}
\noindent\rule{15cm}{1.5pt}\\
{\Large \bf Supplementary Material\\
Fast and Simple Optimization for Poisson Likelihood Models\\}
\noindent\rule{15cm}{1.5pt}\\
\end{center}
\paragraph{Outline.} The appendix  gives a self-contained presentation and analysis of the proposed Composite Mirror Prox algorithm and the Randomized Block Mirror Prox algorithm. We also provide detailed illustration and theoretical analysis of the algorithms when applied to the two point processing applications as well as another new application,  Poisson imaging (which is not discussed in the main text due to space limitation). 

\section{Relation between the Composite Saddle Point Problem and Variational Inequality}
By incorporating the saddle point representation, the penalized Poisson likelihood model (\ref{model:problemofinterest}) becomes a special case of composite saddle point problem
 \begin{equation}
\min_{u_1\in U_1}\max_{u_2\in U_2}\left[\phi(u_1,u_2)+\Psi_1(u_1)-\Psi_2(u_2)\right]
\end{equation}
that satisfies
\begin{itemize}
\item $U_1\subset E_1$ and $U_2\subset E_2$ are nonempty closed convex sets in Euclidean spaces $E_1,E_2$;
\item $\phi(u_1,u_2)$ is a convex-concave function on $U_1\times U_2$ with  Lipschitz continuous gradient;
\item $\Psi_1:U_1\to\bR$ and $\Psi_2:U_2\to\bR$ are convex functions, perhaps non-smooth, but ``fitting'' the domains $U_1$, $U_2$ in the following
sense: for $i=1,2$, we can equip $E_i$ with a norm $\|\cdot\|_{(i)}$, and $U_i$  with a compatible with this norm distance generating function(d.g.f.)
$\omega_i(\cdot)$ in such a way that optimization problems of the form for any $\alpha>0,\beta>0$ and input $\xi\in E_i$
\begin{equation}\label{auxil}
\min_{u_i\in U_i}\left\{\alpha \omega_i(u_i)+\langle \xi,u_i\rangle+\beta \Psi_i(u_i) \right\}
\end{equation}
are easy to solve.
\end{itemize}

In the case of Poisson likelihoods, we have the embedding Euclidean spaces $E_1=\bR^n$, $E_2=\bR^m$, and the closed convex domains  $U_1=\bR_+^n$, $U_2=\bR_+^m$, the convex-concave function $\phi(u_1,u_2)=s^Tu_1-u_2^TAu_1+c_0$,  and  two convex non-smooth terms  $\Psi_1(u_1)=h(u_1)$,  $\Psi_2(u_2)=-\sum_{i=1}^mc_i\log(u_{2,i})$. Particularly, we could equip $E_2$ with distance generating function $\omega_2(u_2)=\frac{1}{2}\|u_2\|_2^2$ w.r.t. the $L_2$ norm $\|\cdot\|_2$ and equip $E_1$ with the distance generating function $\omega_1(u_1)=\sum_{j=1}^n u_{1,j}\log(u_{1,j})$ w.r.t. the $L_1$ norm $\|\cdot\|_1$. 

We can write the composite saddle point problem equivalently as 
\begin{equation}\label{equivalently}
\min_{x^1=[u_1;v_1]\in X_1}\max_{x^2=[u_2;v_2]\in X_2}\left[\Phi(u_1,v_1;u_2,v_2)=\phi(u_1,u_2)+v_1-v_2\right]
\end{equation}
where $X_i=\{x_i=[u_i;v_i]\in E_i\times \bR: u_i\in U_i,v_i\geq\Psi_i(u_i)\}$ for $i=1,2$.

Finding a saddle point $x=[x^1;x^2]$ of $\Phi$ on $X_1\times X_2$ reduces to solving the associated variational inequality (V.I.),
\begin{equation}
\text{Find } x_*=[x^1_*;x^2_*]\in X_1\times X_2:\; \langle F(x_*),x-x_*\rangle \geq 0,\; \forall x\in X_1\times X_2
\end{equation}
where 
\begin{equation*}
F(u=[u_1;u_2],v=[v_1;v_2])=\left[F_u(u)=[\nabla_{u_1}\phi(u_1,u_2);-\nabla_{u_2}\phi(u_1,u_2)];F_v=[1;1]\right].
\end{equation*}
Note that since  $\Phi$ is convex-concave and continuously differentiable, the operator $F$ is monotone and Lipschitz continuous.

\section{Revisiting the Composite Mirror Prox Algorithm}
We intend to process the above type of composite saddle point problem by a simple prox-method - composite Mirror Prox algorithm, 
as established in~\cite{he2015mirror,he:harch:2015}. 
The algorithm is designed to solve variational inequalities with the above structure, allowing to cover the composite saddle point problem as a special case. 

\paragraph{Variational inequality with composite structure.} Essentially, we aim at solving the variational inequality VI$(X,F)$:
$$ \text{Find } x_*\in X: \langle F(x),x-x_*\rangle\geq 0, \forall x\in X$$
with domain $X$ and operator $F$ that satisfy the conditions below:
\begin{enumerate}
\item Set $X\subset E_u\times E_v$ is closed convex and its projection $PX=\{u: x=[u;v]\in X\}\subset U$, where $U$ is convex and closed, $E_u,E_v$ are Euclidean spaces;
\item The function $\omega(\cdot):U\to \bR$ is continuously differentiable 
and also 1-strongly convex w.r.t. some norm $\|\cdot\|$, that is 
\begin{equation*}
\omega(u')\geq\omega(u)+\langle\nabla \omega(u),u'-u\rangle+ \frac{1}{2}\|u'-u\|^2,\forall u,u'\in U;
\end{equation*}
\item The operator $F(x=[u,v]): X\to E_u\times E_v$ is monotone and of form $ F(u,v)=[F_u(u);F_v]$
with $F_v\in E_v$ being a constant and $F_u(u)\in E_u$ is Lipschitz continuous, i.e. for some $L>0$,
\begin{equation*}
\begin{array}{c}
\forall u,u'\in U: \|F_u(u)-F_u(u')\|_*\leq L\|u-u'\|
\end{array}\end{equation*}
where $\|\cdot\|_*$ is the dual norm to $\|\cdot\|$. 
\item The linear form $\langle F_v,v\rangle$ of $[u;v]\in E_u\times E_v$ is bounded from below on $X$ and is coercive on $X$ w.r.t. $v$.
\end{enumerate}

\paragraph{Composite Mirror Prox.} The algorithm converges at a rate of $O(L/t)$ and works as follows
 \begin{algorithm}[H]
\caption{Composite Mirror Prox Algorithm}
\label{alg:CoMP}
\begin{algorithmic}
\STATE \textbf{Input:} stepsizes $\gamma_t>0$, $t = 1,2, \ldots$
\STATE Initialize $x^1=[u^1;v^1]\in X$
\FOR{$t=1,2,\ldots,T$}
\STATE \text{}\\[-0.5cm]
\begin{equation}\label{eq:MPstep}
\begin{array}{rcl}
y^{t}:=[\hat u^t; \hat v^t]&=&P_{x^t}(\gamma_t F(x^t))=P_{x^t}(\gamma_t[ F_u(u^t);F_v])\\
x^{t+1}:=[u^{t+1};v^{t+1}]&=&P_{x^t}(\gamma_t F(y^t))=P_{x_t}(\gamma_t[ F_u(\hat u^t);F_v])
\end{array}
\end{equation}
\ENDFOR\\
\STATE \textbf{Output:} $x_{T}:=[{u}_{T};{v}_{T}] ={(\sum_{t=1}^t\gamma_t)}^{-1}{\sum_{t=1}^t \gamma_t y^{t}}$
\end{algorithmic}
\end{algorithm}
where the prox-mapping is defined by 
\begin{align}
P_{x_0}(\xi) =\Argmin_{x:=[u;v]\in X} \left\{\langle \xi,x\rangle +V_\omega(u,u_0)\right\} \label{eq:prox}
\end{align}
for any $x_0=[u_0;v_0]$ and $\xi\in E_u\times E_v$ and Bregman distance $V_\omega(u,u_0)=\omega(u)-\omega(u_0)-\langle\omega'(u_0),u-u_0\rangle$. 

\begin{theorem} \cite{he2015mirror} Under the above situation and under the choice of stepsizes $0<\gamma_t\leq 1/L$, we have for any set $X'\subset X$, it holds
\begin{equation}\label{eq:mainCMP}
\epsilonvi(x_T|X',F):=\sup_{x\in X'}\langle F(x), x_T-x\rangle \leq \frac{\sup_{u:[u,v]\in X'}V_\omega(u,u^1)}{\sum_{t=1}^T\gamma_t}
\end{equation}
\end{theorem}

For composite saddle point problems as described in Section 3, the above algorithm reduces to Algorithm~1 and we immediately arrive at the convergence results stated in Lemma~\ref{lem:main}.

\section{Composite Mirror Prox Algorithm for Penalized Poisson Regression}\label{sec:appendix_PoissonReg}
The crux of our approach is to work on the saddle point representation of the penalized Poisson regression  (\ref{model:problemofinterest}), which is given by 
\begin{eqnarray}
\min_{x\in \bR_+^n} \max_{y\in\bR_{++}^m} s^Tx-y^TAx+\sum_{i=1}^m c_i \log(y_i)+h(x)+c_0.
\end{eqnarray}
The resulting saddle point problem falls exactly into the regime of composite saddle point problem as described in Section 3. Invoking  Lemma~\ref{lem:main} with the specific mixed proximal setups, we can easily derive the error bounds as stated in Theorem~\ref{prop:Poisson}. To avoid the redundancy, we omit the proof here. In the following, we provide the proof for the following simple fact. 

\begingroup
\def\thetheorem{\ref{prop:domain}}
\begin{proposition}
The optimal solution $x_*$ to the problem in (\ref{model:problemofinterest}) satisfies
\begin{equation}\label{eq:relation}
\begin{array}{c}
 s^Tx_*+h(x_*)= \sum_{i=1}^m c_i.
\end{array}
\end{equation}
\end{proposition}
\addtocounter{theorem}{-1}
\endgroup

\begin{proof} This is because, for any $t>0$, $tx_*$ is a feasible solution and the objective at this point is 
$\phi(t):=L(t x_*)+h(tx_*) =t(s^Tx_*+h(x_*))-\sum_{i=1}^mc_i\log (a_i^Tx_*)-\log(t)\sum_{i=1}^mc_i.
$
By optimality, $\phi'(1)=0$, i.e. (\ref{eq:relation}) holds.
\end{proof}

\section{Fully Randomized Block Mirror Prox Algorithm}\label{sec:appendix_RBMP}
We propose a randomized block-decomposition variant of Composite Mirror Prox, that is appropriate
for large-sample datasets. Block-coordinate optimization has received much attention and success recently for solving high-dimensional problems.  However, to the best of our knowledge, 
this is the first time that a randomized block-coordinate variant of Mirror Prox is developed.

\paragraph{Variational inequality with block structure.} We consider the above variational inequality with block structure, i.e. 
$$X=X_1\times X_2\times\cdots\times X_b$$
where $X_k$ are closed convex sets. More specifically, we consider the situation
\begin{enumerate}
\item For $k=1,\ldots,b$, $X_k$ is closed convex and its projection $PX_k=\{u_k: x_k=[u_k;v_k]\in X_k\}\subset U_k$, where $U_k$ is convex and closed;
\item For $k=1,\ldots,b$, the function $\omega_k(\cdot):U_k\to \bR$ is continuously differentiable 
and also 1-strongly convex w.r.t. some norm $\|\cdot\|_k$, that is 
\begin{equation*}
\omega_k(u')\geq\omega_k(u)+\langle\nabla \omega_k(u),u'-u\rangle+ \frac{1}{2}\|u'-u\|^2,\forall u,u'\in U_k;
\end{equation*}
This defines the Bregman distance
$V_k(u,u_0)=\omega(u)-\omega(u_0)-\langle\omega'(u_0),u-u_0\rangle$ for any $u,u_0\in U_k$.
\item The operator $F(x=[u,v])=[F_u(u);F_v]$ with $F_u(u)=[F_{u,1}(u);\ldots;F_{u,b}(u)]$ and $F_v=[F_{v,1};\ldots; F_{v,b}]$, and assume for any $k=1,\ldots, b$, 
\begin{equation*}
\begin{array}{c}
\|F_{u,k}(u)-F_{u,k}(u')\|_{k,*}\leq L_k\|u_k-u_k'\|_k, \;   \forall u,u'\in U_k \text{ and } u_l=u'_l, l\neq k
\end{array}\end{equation*}
\item For $k=1,\ldots,b$, the linear form $\langle F_{v,k} ,\cdot \rangle$  is bounded from below and coercive on $U_k$ .
\end{enumerate}

\paragraph{Randomized block Mirror Prox.} We present the algorithm below. To the best of our knowledge, this is the first time, such modification of the Mirror Prox algorithm is developed.  
 \begin{algorithm}[H]
\caption{Randomized Block Mirror Prox Algorithm }
\label{alg:RBMP}
\begin{algorithmic}
\STATE \textbf{Input:} stepsizes $\gamma_t>0$, $t = 1,2, \ldots$
\STATE Initialize $x^1=[u^1;v^1]\in X$
\FOR{$t=1,2,\ldots,T$}
\STATE  Pick $k_t$ at random in $\{1,...,b\}$
\STATE Update 
$y^{t}:=[\hat u^t; \hat v^t]=\left\{
\begin{array}{ll}
P^{k}_{x^t}(\gamma_t[ F_{u,k}(u^t);F_{v,k}]), & k=k_t\\
x^{t}_k, & k\neq k_t
\end{array}\right.
$
\STATE Update
$x^{t+1}:=[u^{t+1}; v^{t+1}]=\left\{
\begin{array}{ll}
P^{k}_{x^t}(\gamma_t[ F_{u,k}(\hat u^t);F_{v,k}]), & k=k_t\\
x^{t}_k, & k\neq k_t
\end{array}\right.
$
\ENDFOR\\
\STATE \textbf{Output:} $x_{T}:=[{u}_{T};{v}_{T}] ={(\sum_{t=1}^t\gamma_t)}^{-1}{\sum_{t=1}^t \gamma_t y^{t}}$
\end{algorithmic}
\end{algorithm}
Unlike the composite Mirror Prox algorithm, the new algorithm randomly pick one block to update at each iteration, which significantly reduces the iteration cost. We discuss the main convergence property of the above algorithm.  For simplicity, we consider the simple situation where the index of block is selected according to a uniform distribution. The analysis could be extended to non-uniform distribution; we leave this for future work.

\paragraph{Convergence analysis.}   
We have the following result
\begin{theorem}\label{thm:RBMP}
Assume that the sequence of step-sizes $(\gamma_t)$ in the above algorithm satisfy $0<\gamma_t\max_{k=1,\ldots,b} L_k\leq 1$. 
Then we have 
\begin{equation}\label{vigap}
\forall z\in X, \; \bE[\langle F(z), x_T-z\rangle]\leq \frac{b\cdot \sup_{[u;v]\in X}\sum_{k=1}^bV_k(u_k,u^1_k)}{\sum_{t=1}^T\gamma_t}.
\end{equation}
In particular, when $\gamma_t\equiv \frac{1}{\max_{k=1,\ldots,b} L_k}$, we have
\begin{equation}
\forall z\in X,\; \bE[\langle F(z), x_T-z\rangle]\leq \frac{b\cdot \sup_{[u;v]\in X}\sum_{k=1}^bV_k(u_k,u^1_k) \cdot\max_{k=1,\ldots,b} L_k}{T}.
\end{equation}
\end{theorem}
\begin{proof}
The proof follows a similar structure to the ones in~\cite{he2015mirror,he:harch:2015}. For all $u,u',w\in U$, we have the so-called three-point identity or Generalized Pythagoras theorem
\begin{equation}\label{threetermid}
\langle V'(u', u),w-u'\rangle =V(w,u)-V(w, u')-V(u', u).
\end{equation}
For $x=[u;v]\in X,\;\xi=[\eta;\zeta]$, $\epsilon\geq0$,  let $[u';v']\in P_x(\xi)$. By definition, for all $[s;w]\in X$, the inequality holds
\[
\langle \eta+V'(u',u),u'-s\rangle+\langle\zeta,v'-w\rangle \le 0,
\]
which by \rf{threetermid} implies that
\begin{equation}\label{prox_lemma}
\langle \eta,u'-s\rangle+\langle\zeta,v'-w\rangle \le \langle V'(u', u),s-u'\rangle=
V(s,u)-V(s, u')-V(u',u).
\end{equation}
For simplicity, let use denote $k=k_t$ as the random index at iteration $t$ and let use denote $V_k(u'_k,u_k)=V_k(u',u)$. When applying \rf{prox_lemma} with $X=X_k$ and $V=V_k$, and $[u;v]=[u^{t}_k;v^{t}_k]=x_k^{t}$, $\xi=\gamma_{t} F_k(x^{t})=[\gamma_{t} F_{u,k}(u^{t});\gamma_{t} F_{v,k}]$, $[{u}';{v}']=[\hat{u}_k^{t};\hat v_k^{t}]=y_k^{t}$, and $[s;w]=[u_k^{{t}+1};v_k^{{t}+1}]=x_k^{{t}+1}$, we obtain
\begin{equation}
\label{prox100}
\gamma_{t} [\langle F_{u,k}(u^t),\hat{u}_k^{t}-u_k^{{t}+1}\rangle+\langle F_{v,k},\hat{v}_k^{t}-v_k^{{t}+1}\rangle]\le V_{k}(u^{{t}+1}, u^{t})-V_{k}(u^{{t}+1},\hat{u}^{t})-V_{k}(\hat{u}^{{t}},u^{t})\; ; 
\end{equation}
and applying \rf{prox_lemma} with $[u;v]=x_k^{t}$, $\xi=\gamma_{t} F_k(y^{t})$, $[{u}';{v}']=x_k^{{t}+1}$, and $[s_k;w_k]\in X_k$ we get
\begin{equation}
\label{prox101}
\gamma_{t} [\langle F_{u,k}(\hat{u}^{t}),u_k^{{t}+1}-s_k\rangle+\langle F_{v,k},v_k^{{t}+1}-w_k\rangle]\le V_{k}(s,u^{t})-V_{k}(s, u^{{t}+1})-V_{k}(u^{{t}+1},u^{t}) \; .
\end{equation}
Adding \rf{prox101} to \rf{prox100}, we obtain for every $z=[s;w]\in X$
\begin{equation}
\label{prox102}
\gamma_{t} \langle F_k(y^{t}),y_k^{t}-z_k\rangle \leq V_{k}(s, u^{t})-V_{k}(s, u^{{t}+1})+\sigma_{t,k}
\end{equation}
where 
\begin{equation*}
\sigma_{t,k} := \gamma_{t}\langle F_{u,k}(\hat{u}^{t})-F_{u,k}(u^{t}),\hat{u}_k^{t}-u_k^{{t}+1}\rangle-V_{k}(u^{{t}+1}, \hat{u}^{t})-V_{k}(\hat{u}^{{t}}, u^{t}) \; .
\end{equation*}

Due to the strong convexity, with modulus 1, of $V_k(\cdot,u)$ w.r.t. $\|\cdot\|_k$, we have for all $u,\hat{u}$
\begin{equation*}
V_k(\hat{u}, u)\geq {1\over 2}\|u_k-\hat{u}_k\|_k^2 \; .
\end{equation*} 
Therefore,
\bse
\sigma_{t,k}&\leq& \gamma_{t}\|F_{u,k}(\hat{u}^{t})-F_{u,k}(u^{t})\|_{k,*}\|\hat{u}^{t}_k-u^{{t}+1}_k\|_k-\half\|\hat{u}_k^{t}-u_k^{{t}+1}\|_k^2-\half\|u^{t}_k-\hat{u}_k^{t}\|_k^2\\
&\leq&
\half\left[\gamma_{t}^2\|F_{u,k}(\hat{u}^{t})-F_{u,k}(u^{t})\|_{k,*}^2-\|u_k^{t}-\hat{u}_k^{t}\|_k^2\right]\\
&\leq&
\half\left[\gamma_{t}^2[L_k\|\hat{u}_k^{t}-u_k^{t}\|_k]^2-\|u_k^{t}-\hat{u}_k^{t}\|_k^2\right] \\
&\leq & 0 
\ese
where the last inequality follows from the condition  $\gamma_t\max_{k=1,\ldots,b} L_k\leq 1$. 

Let $V(u',u) = \sum_{k=1}^b V_k(u'_k,u_k)$ for any $u,u'\in U$. Then,  we have 
$V_{k}(s, u^{t})-V_{k}(s, u^{{t}+1}) = V(s,u^t)-V(s,u^{{t}+1})$. 
The inequality ~(\ref{prox102}) now becomes
\begin{equation}
\label{prox103}
\gamma_{t} \langle F_k(y^{t}),y_k^{t}-z_k\rangle \leq V(s,u^t)-V(s,u^{{t}+1}).
\end{equation}
Let us denote $V_k$ as the projection matrix such that $Q_k^Tx= x_k, k = 1,\ldots, n$. Summing up inequalities~(\ref{prox103}) over ${t}=1,2,...,T$, and taking into account that $V(s, u^{T+1})\geq0$, we have
\begin{equation}
\label{prox104}
\sum_{t=1}^T\gamma_{t} \langle Q_{k_t}F_{k_t}(y^{t}),y^{t}-z\rangle \leq V(s,u^1)
\end{equation}

Conditioned on $\{k_1,\ldots, k_{t-1}\}$ , we have 
$\bE_{k_t}[ \langle Q_{k_t}F_{k_t}(y^{t}),y^{t}-z\rangle]=\frac{1}{b}\sum_{k=1}^b \langle Q_k F_k(y^t),y^t-z\rangle =\frac{1}{b}\langle F(y^{t}),y^{t}-z\rangle $. Taking expectation over $\{k_1,\ldots, k_T\}$, we finally conclude that for all $z=[s;w]\in X$,
\begin{align*}
\bE[\sum_{{t}=1}^T\lambda_T^{t}\langle F(y^{t}),y^{t}-z\rangle] 
&\leq{b\cdot V(s,u^1)\over
\sum_{{t}=1}^T\gamma_{t}}, \text{ where }\lambda_T^{t}=\left(\sum_{i=1}^T\gamma_i \right)^{-1}\gamma_{t} \; .
\end{align*}
Invoking the monotonicity of $F$, we end up with (\ref{vigap}). 
\end{proof}
%%%

\paragraph{Discussion.} Assume that $X\subset \bR^n$ and the cost of computing the full gradient is $O(n)$, then here is the comparison between the (batch) composite Mirror Prox and the randomized block variant.  \\

 \begin{table}[!ht]
 \begin{center}
{\small
\caption{ composite Mirror Prox : batch vs randomized block}
\begin{tabular}{l|c|c|c|c}
\hline
 Algorithm &type & guarantee &  convergence & average iteration cost\\
\hline\hline
Composite Mirror Prox & batch & primal and dual & $O( L/t)$ & $O(n)$ \\
\hline
Randomized Block Mirror Prox &stoch. &sad. point gap & $O(b\max_{k}L_k/t)$ & $O(n/b)$\\
\hline
\end{tabular}}
\end{center}
\end{table}

\section{Partially Randomized Block Mirror Prox Algorithm}\label{sec:appendix_partRBMP}
There is clearly a delicate tradeoff between the fully randomized algorithm and fully batch algorithm. The optimal tradeoff for our purpose actually lies in between. Indeed, to further improve the overall efficiency, we might prefer to keep updating some variables (those more important and low-dimensional ones) at iteration, while randomly select from other variables (those less important and high-dimensional ones) to update. The problem of our interest - saddle point reformulation (\ref{model:problemsimplified}), is exactly under such situation. 

We introduce a new partially randomized block-decomposition scheme, that accommodates partially randomized block updating rules. If one keeps updating the primal variable $x$ at each iteration, while only updating a random block for the dual variable $y$, one gets a more efficient scheme than both the fully randomized and the fully batch ones.

\paragraph{Variational inequality with partial block structure.} We consider the above variational inequality with block structure, i.e. 
$$X=X_0\times (X_1\times X_2\times\cdots\times X_b)$$
where $X_k$ are closed convex sets, $k=0,1,\ldots,b$. More specifically, we consider the situation
\begin{enumerate}
\item For $k=0, 1,\ldots,b$, $X_k$ is closed convex and its projection $PX_k=\{u_k: x_k=[u_k;v_k]\in X_k\}\subset U_k$, where $U_k$ is convex and closed;
\item For $k=0,1,\ldots,b$, the function $\omega_k(\cdot):U_k\to \bR$ is continuously differentiable 
and also 1-strongly convex w.r.t. some norm $\|\cdot\|_k$, and defines the Bregman distance $V_k(u',u)$.
\item The operator $F(x=[u,v])=[F_u(u);F_v]$ with $F_u(u)=[F_{u,1}(u);\ldots;F_{u,b}(u)]$ and $F_v=[F_{v,1};\ldots; F_{v,b}]$, and assume for any $k=1,\ldots, b$, 
\begin{equation*}
\begin{array}{c}
\|F_{u,k}(u)-F_{u,k}(u')\|_{k,*}\leq L_k\|u_k-u_k'\|_k , \;   \forall u,u'\in U_k \text{ and } u_l=u'_l, l\neq k\\
\|F_{u,k}(u)-F_{u,k}(u')\|_{k,*}\leq G_k\|u_0-u_0'\|_0, \;   \forall u,u'\in U_k \text{ and } u_l=u'_l, l\neq 0\\
\|F_{u,0}(u)-F_{u,0}(u')\|_{0,*}\leq G_k\|u_k-u_k'\|_k, \;   \forall u,u'\in U_k \text{ and } u_l=u'_l, l\neq k\\
\|F_{u,0}(u)-F_{u,0}(u')\|_{0,*}\leq L_0\|u_0-u_0'\|_0, \;   \forall u,u'\in U_k \text{ and } u_l=u'_l, l\neq 0
\end{array}\end{equation*}
\item For $k=0,1,\ldots,b$, the linear form $\langle F_{v,k} ,\cdot \rangle$  is bounded from below and coercive on $U_k$ .
\end{enumerate}

\paragraph{Partially randomized block Mirror Prox.} We present the algorithm below. At each iteration, the algorithm update the block $x_0$ and another block randomly selected from $\{x_1,x_2,\ldots,x_b\}$. 
 \begin{algorithm}[H]
\caption{Partially Randomized Block Mirror Prox Algorithm }
\label{alg:RBMP}
\begin{algorithmic}
\STATE \textbf{Input:} stepsizes $\gamma_t>0$, $t = 1,2, \ldots$
\STATE Initialize $x^1=[u^1;v^1]\in X$
\FOR{$t=1,2,\ldots,T$}
\STATE  Pick $k_t$ at random in $\{1,...,b\}$
\STATE Update 
$y^{t}:=[\hat u^t; \hat v^t]=\left\{
\begin{array}{ll}
P^{k}_{x^t}(\gamma_t[ F_{u,k}(u^t);F_{v,k}]), & k\in\{k_t\cup 0\}\\
x^{t}_k, &  k\notin\{k_t\cup 0\}
\end{array}\right.
$
\STATE Update
$x^{t+1}:=[u^{t+1}; v^{t+1}]=\left\{
\begin{array}{ll}
P^{k}_{x^t}(\gamma_t[ F_{u,k}(\hat u^t);F_{v,k}]), & k\in\{k_t\cup 0\}\\
x^{t}_k, & k\notin\{k_t\cup 0\}
\end{array}\right.
$
\ENDFOR\\
\STATE \textbf{Output:} $x_{T}:=[{u}_{T};{v}_{T}] ={(\sum_{t=1}^t\gamma_t)}^{-1}{\sum_{t=1}^t \gamma_t y^{t}}$
\end{algorithmic}
\end{algorithm}

\paragraph{Convergence analysis.} We have the following result
\begin{theorem}\label{thm:RBMP}
Assume that the sequence of step-sizes $(\gamma_t)$ in the above algorithm satisfy 
\bse
\gamma_t&>&0\\
\gamma_t^2(2bL_k^2+G_k^2)-b&\leq &0 ,\forall k=1,2,\ldots, b\\
\gamma_t^2(2L_0^2+2bG_k^2)-1&\leq&0,\forall k=1,2,\ldots, b.
\ese

Then we have for any $z\in X$
\begin{equation}\label{vigap2}
\bE[\langle F(z), x_T-z\rangle]\leq \frac{\sup_{[u;v]\in X} \{V_0(u_0,u^1_0)+b\sum_{k=1}^bV_k(u_k,u^1_k)\}}{\sum_{t=1}^T\gamma_t}.
\end{equation}
\end{theorem}

\begin{proof} Similar to previous proof, we have for (\ref{prox102}) for $k=\{k_t\cup 0\}$, i.e. 
\begin{equation}
\label{prox104}
\gamma_{t} \langle F_k(y^{t}),y_k^{t}-z_k\rangle \leq V_{k}(s, u^{t})-V_{k}(s, u^{{t}+1})+\sigma_{t,k}, \forall z=[s;w]\in X
\end{equation}
where 
\begin{equation*}
\sigma_{t,k} := \gamma_{t}\langle F_{u,k}(\hat{u}^{t})-F_{u,k}(u^{t}),\hat{u}_k^{t}-u_k^{{t}+1}\rangle-V_{k}(u^{{t}+1}, \hat{u}^{t})-V_{k}(\hat{u}^{{t}}, u^{t}) \; .
\end{equation*}

We have for $k=k_t$,
\bse
\sigma_{t,k}&\leq& \gamma_{t}\|F_{u,k}(\hat{u}^{t})-F_{u,k}(u^{t})\|_{k,*}\|\hat{u}^{t}_k-u^{{t}+1}_k\|_k-\half\|\hat{u}_k^{t}-u_k^{{t}+1}\|_k^2-\half\|u^{t}_k-\hat{u}_k^{t}\|_k^2\\
&\leq&
\half\left[\gamma_{t}^2\|F_{u,k}(\hat{u}^{t})-F_{u,k}(u^{t})\|_{k,*}^2-\|u_k^{t}-\hat{u}_k^{t}\|_k^2\right]\\
&\leq&
\half\left[\gamma_{t}^2[2L_k^2\|\hat{u}_k^{t}-u_k^{t}\|_k+2G_k^2\|\hat{u}_0^{t}-u_0^{t}\|_0]-\|u_k^{t}-\hat{u}_k^{t}\|_k^2\right] \\
\ese
and also 
\bse
\sigma_{t,0}
&\leq&
\half\left[\gamma_{t}^2[2L_0^2\|\hat{u}_0^{t}-u_0^{t}\|_0+2G_k^2\|\hat{u}_k^{t}-u_k^{t}\|_k]-\|u_0^{t}-\hat{u}_0^{t}\|_0^2\right] \\
\ese
Let Let $V(u',u) =V_0(u'_0,u_0)+b\cdot\sum_{k=1}^b V_k(u'_k,u_k)$ for any $u,u'\in U$. Then,  we have 
$$V_ 0(s,u^{t})-V_0^(s,u^{t+1})+b\cdot [V_{k}(s, u^{t})-V_{k}(s, u^{{t}+1})]= V(s,u^t)-V(s,u^{{t}+1}).$$
Summing up (\ref{prox104}) with $k=k_t$ and $k=0$, we have 
\bse
\gamma_{t}\langle Q_0F_0(y^t)+ b Q_kF_k(y^{t}),y^t-z\rangle &\leq &V(s,u^t)-V(s,u^{{t}+1}) +b\sigma_{t,k}+\sigma_{t,0},\\
&\leq & V(s,u^t)-V(s,u^{{t}+1})
\ese
where the last inequality  follows from the condition  
$$\gamma_t\leq \min_{k=1,\ldots,b}\left\{\frac{1}{\sqrt{2L_0^2+2bG_k^2}},\frac{1}{\sqrt{2L_k^2+2G_k^2/b}}\right\}.$$

 Summing up inequalities~(\ref{prox103}) over ${t}=1,2,...,T$, and taking expectation over $\{k_1,\ldots, k_T\}$, we finally conclude that for all $z=[s;w]\in X$,
\begin{align*}
\bE[\sum_{{t}=1}^T\lambda_T^{t}\langle F(y^{t}),y^{t}-z\rangle] 
&\leq{V(s,u^1)\over
\sum_{{t}=1}^T\gamma_{t}}, \text{ where }\lambda_T^{t}=\left(\sum_{i=1}^T\gamma_i \right)^{-1}\gamma_{t} \; .
\end{align*}
Invoking the monotonicity of $F$, we end up with (\ref{vigap2}). 
\end{proof}

\paragraph{Convergence analysis for penalized Poisson regression.} When solving the saddle point reformulation (\ref{model:problemsimplified}) with the randomized block Mirror Prox algorithm~\ref{alg:randomizedCMP} presented in Section~\ref{sec:extension}, we specifically have $L_k=L_0=0$ and $G_k=\max_{x\in\bR^n_+:\|x\|_x\leq 1} \{\|A_kx\|_2\}$, which gives rise to Theorem~\ref{prop:RandomizedCMP}.

\section{Application: Positron Emission Tomography}\label{sec:appendix_PET}
PET imaging plays an important role in nuclear medicine for detecting cancer and metabolic changes in human organ. Image reconstruction in PET has a long history of being treated as a Poisson likelihood model~\cite{ben2001ordered, harmany2012spiral}. To estimate the density of radioactivity within an organ corresponds to solving the convex optimization problem
\begin{equation}\label{model:PET}
\begin{array}{c}
\min_{x\in \bR_+^n} \sum_{i=1}^m\left[[Ax]_i-w_i\log([Ax]_i)\right]
\end{array}
\end{equation}
where $A$ refers to the likelihood matrix known from the geometry of detector, and $w$ refers to the vector of events detected with $w_i\sim \textit{Poisson} ([Ax]_i), 1\leq i\leq m.$ Clearly, this is a special case of Poisson regression  (\ref{model:problemofinterest}). For simplicity, we shall not consider any penalty  term for this application. 

\paragraph{Saddle Point Reformulation} Invoking the optimality conditions for the above problem, we have
\begin{equation*}
x_j\sum_{i=1}^m \left[a_{ij} - w_i\frac{a_{ij}}{[Ax]_i}\right]=0, \forall j =1,\ldots, n,
\end{equation*}
whence, summing over $j$ and taking into account that $A$ is stochastic, we get\footnote{Note that this is essentially a special case revealed by Remark 1 in previous section.} 
\begin{equation*}
\begin{array}{c}
\sum_{j=1}^n x_j=\sum_{i=1}^m w_i=:\theta.
\end{array}
\end{equation*}
We loose nothing by adding to problem (\ref{model:PET}) the equality constraints $\sum_{j=1}^n x_j=\theta$. Invoking the saddle point reformulation in the previous section, solving the PET recovery problem (\ref{model:PET}) is equivalent to solving the convex-concave saddle point problem:
\begin{eqnarray}\label{model:saddlePET}
\min_{{x\in \bR_+^n: \sum_{j=1}^n x_j=\theta}} \max_{y\in\bR_{++}^m}-y^TAx+\sum_{i=1}^m w_i \log(y_i)+\tilde\theta
\end{eqnarray}
where $\tilde\theta=2\theta-\sum_{i=1}^m\omega_i\log(\omega_i)$ is a constant. 

\paragraph{Composite Mirror Prox algorithm for PET} Noting that the domain over $x$ is a simplex, a good choice for proximal setup is to use the entropy function $\omega(x)=\sum_{j=1}^nx_i\log(x_j)$. For completeness, we customize the algorithm and provide full algorithmic details for this specific example (\ref{model:saddlePET}).

\begin{algorithm}
\caption{Composite Mirror Prox for PET}
\label{CMPforPET}
\begin{algorithmic}
\STATE 0. Initialize $x^1\in \bR^n_+$, $y^1\in\bR^n_{++}$, $\alpha>0$ and $\gamma_t>0$,
\FOR{$t=1,2,\ldots, T$}
  \STATE 1. Compute\small
\vspace{-2mm}
$$\begin{array}{l}
\vspace{2mm}
\hat x^t_j=x^t_j\exp(-[A^Ty]_j/\alpha), \forall j, \text{and normlize to sum up to } \theta\\
\vspace{2mm} 
\hat y^t_i= \frac{-\gamma_t(a_i^Tx^t-y^t_i)+\sqrt{\gamma_t^2(a_i^Tx^t-y^t_i)^2+4\gamma_tw_i}}{2}, \forall i 
\end{array}$$
\vspace{-4mm}
 \STATE 2. Compute 
\vspace{-2mm}
$$\begin{array}{l}
\vspace{2mm}
x^{t+1}_j=x^t_j\exp(-[A^T\hat y]_j/\alpha), \forall j,\text{and normlize to sum up to } \theta\\
\vspace{2mm}
y^{t+1}_i= \frac{-\gamma_t(a_i^T\hat x^t-y^t_i)+\sqrt{\gamma_t^2(a_i^T\hat x^t-y^t_i)^2+4\gamma_tw_i}}{2},\forall i
\end{array}$$
\vspace{-4mm}
\ENDFOR\\
Output $x_{T} =\frac{\sum_{t=1}^T \gamma_t x^t}{\sum_{t=1}^T\gamma_t}$
\end{algorithmic}
\end{algorithm}

\paragraph{Remark.} Let $x_*$ be the true image. Note that when there is no Poisson noise, $w_i=[Ax_*]_i$ for all $i$. In this case, the optimal solution $y_*$ corresponding to the $y$-component of the saddle point problem (\ref{model:saddlePET}) is given by $y_{*,i}=w_i/[Ax_*]_i=1, \forall i$. Thus, we may hope that under the Poisson noise, the optimal $y_*$ is still close to $1$. Assuming that this is the case, the efficiency estimate for $T$-step composite Mirror Prox algorithm in Algorithm~\ref{CMPforPET} after invoking Proposition~\ref{prop:Poisson} and setting $\alpha=r^2m$ for some $r>0$, will be
$$O(1)\left(\log(n)+\frac{1}{2r^2}\right)\frac{r\theta\sqrt{m}\|A\|_{1\to2}}{T}.$$
Since $A$ is $m\times n$ stochastic matrix, we may hope that the Euclidean norms of columns in $A$ are of order $O(m^{-1/2})$, yielding the efficiency estimate $O(1)\left(\log(n)+\frac{1}{2r^2}\right)\frac{r\theta}{T}.$
Let us look what happens in this model when $x_*$ is ``uniform", i.e. all entries in $x_*$ are $\theta/n$. In this case, the optimal value is $\theta-\theta\log(\theta)+\theta\log(n)$, which is typically of order $O(\theta)$, implying that relative to optimal value rate of convergence is about $O(1/T)$.

\paragraph{Numerical results.} We ran experiments on several phantom images of size $256\times 256$. We built the matrix $A$, which is of size $43530 \times 65536$. To evaluate the efficiency our algorithm, we consider the noiseless situation; hence, the optimal solution and objective value are known. We also compare our algorithm in terms of the relative accuracy, i.e. $(f(x_t)-f_*)/f_*$,  to the Mirror Descent algorithm proposed in~\cite{ben2001ordered} and the Non-monotone Maximum Likelihood  algorithm in~\cite{sra2008non}.  Results are presented in Figure~\ref{fig:CoMPvsMD}.  Fig.\ref{fig:MRI accuracy} corroborates the sublinear convergence of the proposed algorithm; Fig.\ref{fig:CoMPrecoveryMRI} provides mid-slices of recovery images of the algorithm; Fig.\ref{fig:Logan accuracy} shows that our algorithm outperforms both competitors after a small number of iteration. This experiment clearly demonstrates that our composite Mirror Prox is an interesting optimization algorithm for PET reconstruction.
\begin{figure}
\begin{center}
\begin{subfigure}[b]{0.46\textwidth}
                   \includegraphics[width=1\textwidth]{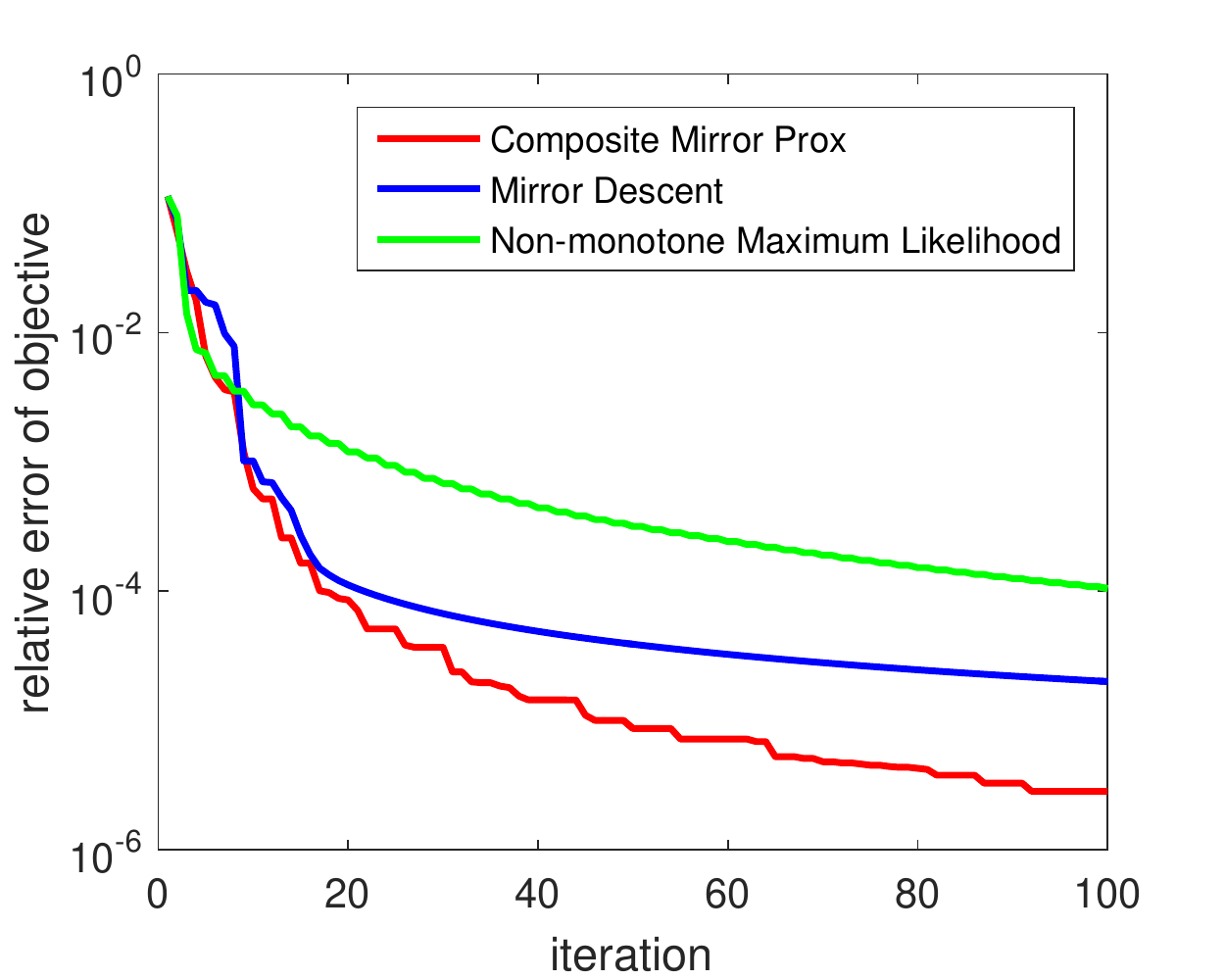}
			\caption{Shepp-Logan image}
                \label{fig:Logan accuracy}
\end{subfigure}
\begin{subfigure}[b]{0.46\textwidth}
                 \includegraphics[width=1\textwidth]{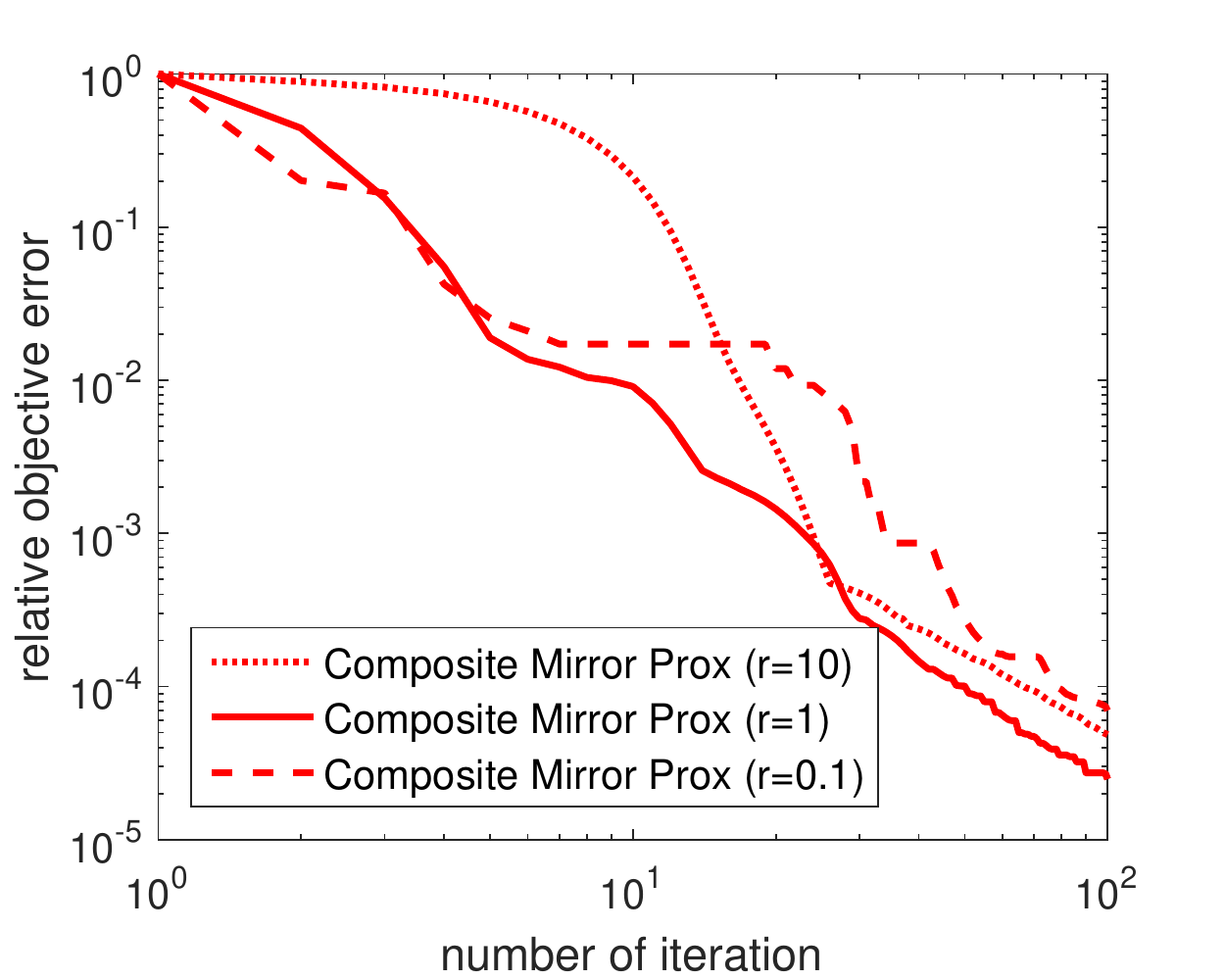}
			\caption{MRI brain image}
                \label{fig:MRI accuracy}
\end{subfigure}
\vspace{0.2cm}

\begin{subfigure}[b]{0.7\textwidth}
			\includegraphics[width=1\textwidth]{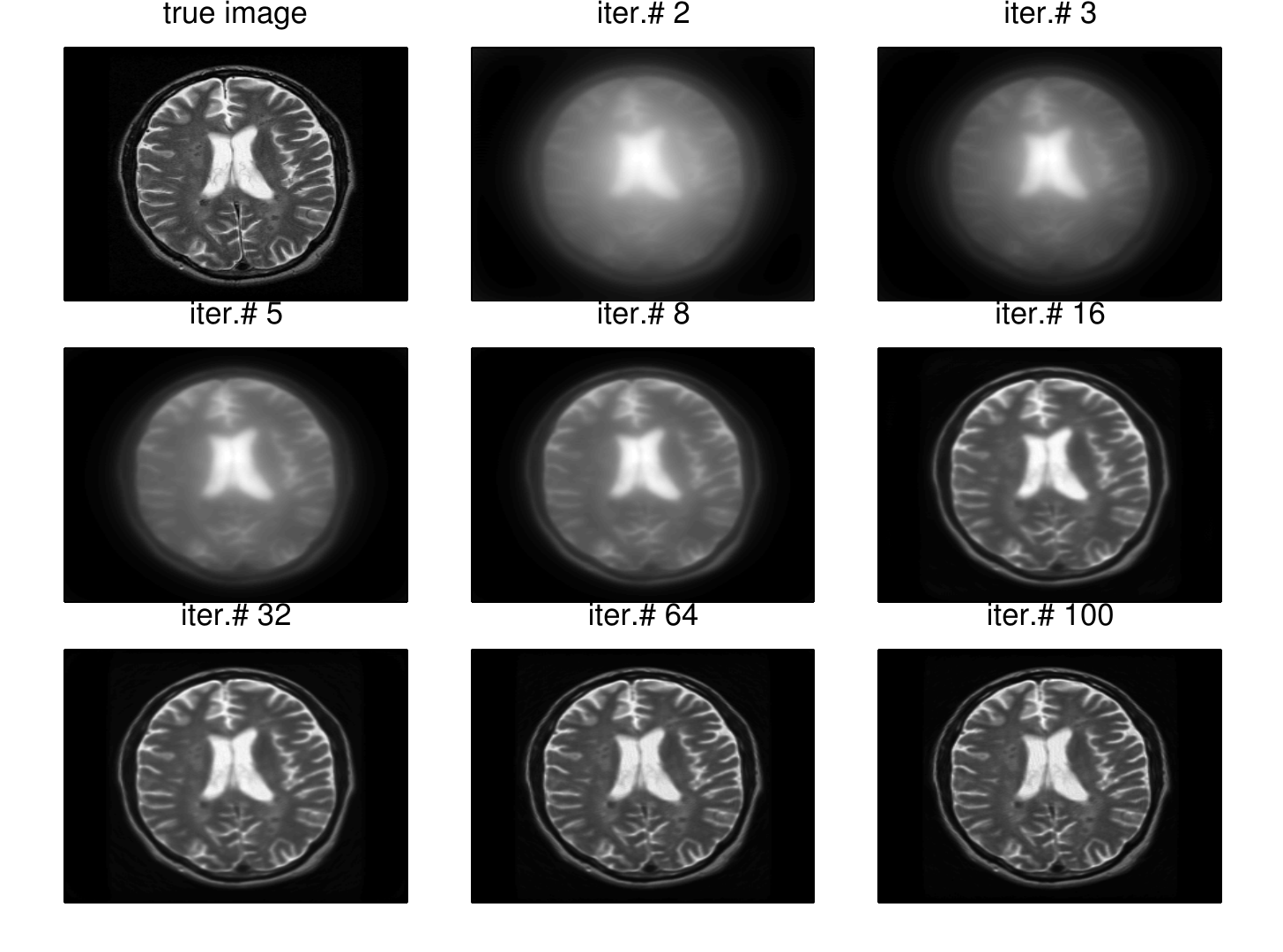}
			\caption{MRI brain image}
			\label{fig:CoMPrecoveryMRI}
\end{subfigure}
\caption{Positron Emission Tomography Reconstruction.\\[-0.5cm] } 

\label{fig:CoMPvsMD}
\end{center}
\end{figure}

\section{Application: Network Estimation}\label{sec:appendix_network}
\paragraph{Problem description.} Given a sequence of events $\{(u_j,t_j)\}_{j=1}^m$, the goal is to estimate the influence matrix among users. We focus on the convex formulation as posed in \cite{zhou2013learning} 
\begin{eqnarray}\label{model:plainHawkes}
&\min\limits_{x\in \bR_+^U, X\in\bR_+^{U\times U}} L(x,X) +\lambda_1\|X\|_1+\lambda_2\|X\|_\nuc\;\;\;
\end{eqnarray}
where 
$L(x,X):= \sum_{u=1}^U [Tx_u+\sum_{j=1}^m X_{uu_j}G(T-t_j)]
-\sum_{j=1}^m\log\big(x_{u_j}+\sum_{k:t_k<t_j}X_{{u_j}{u_k}}g(t_j-t_k)\big)$ is the log-likelihood term, $\|\cdot\|_1$ is the $L_1$ norm  and $\|\cdot\|_\nuc$ is the nuclear norm. Variable $x\geq 0$ denotes the base intensity for all users and matrix $X\geq 0$ the  infectivity matrix. Moreover, $G(t)=\int_{0}^Tg(s)ds$, where the exponential kernel $g(t)=ce^{-ct}$ with $c=1$. 
The function $L(x,X)$ can be simplified to
$$L(x,X)=T\cdot \mathbf{1}^Tx+\mathbf{1}^TXd  -\sum_{j=1}^m\log\left(e_j^Tx+\Tr(A_jX)\right)$$
by setting 
$ d=(d_u)_{u=1,\ldots,U}, e_j=(e_u^j)_{u=1,\ldots, U}, A_j=[a_{uu'}^j]_{u,u'=1,\ldots, U}$ with
\begin{equation*}
\begin{array}{c}
d_u=\sum_{j: u_j=u} G(T-t_j), \quad e_u^j=\left\{\begin{array}{ll} 1,&u=u_j\\ 0,&\text{o.w.}\end{array}\right.,\;
a_{uu'}^j=\left\{\begin{array}{ll}
\sum_{k:t_k<t_j, u_k=u'}g(t_j-t_k), &u=u_j\\
0,&\text{o.w.}
\end{array}\right.
\end{array}
\end{equation*}
for any $ j=1,\ldots, m$. For the sake of simplicity, we will consider the case when $\lambda_2=0$, i.e. no additional regularization term besides the $L_1$ norm. One can more clearly see that the above formulation again falls into the penalized Poisson regression problem in (\ref{model:problemofinterest}).

\paragraph{Saddle point reformulation.} The corresponding saddle point problem is therefore given by 
\begin{eqnarray}\label{model:saddleHawkes}
\min_{x,\in \Delta_x, X\in\Delta_X} \max_{y>0}\quad \underbrace{T\cdot \mathbf{1}^Tx+\mathbf{1}^TXd -\sum_{j=1}^my_j\left(e_j^Tx+\Tr(A_jX)\right)+m}_{\phi(x,X;y)}+\underbrace{\sum_{j=1}^m\log(y_i)}_{-\Psi_2(y)} +\underbrace{\lambda_1\|X\|_1}_{\Psi_1(X)}
\end{eqnarray}
where $\phi(x,X;y)$ is Lipschitz differentiable, $\Psi_1(X)$ and $\Psi_2(y)$ are convex and proximal-friendly.  

\paragraph{Composite Mirror Prox algorithm for network estimation.} Invoking Lemma~\ref{lem:simplefact}, the optimal solution $(x_*,X_*)$ to the above model satisfies
\begin{equation}\label{eq:relation}
T\cdot \mathbf{1}^Tx_*+\mathbf{1}^TX_*d +\lambda_1\|X_*\|_1=m.
\end{equation}
The above observation implies that we are allowed to add to the problem a bounded domain 
$$\cX\subset\{(x,X): x\geq 0, X\geq 0, \,T\cdot \mathbf{1}^Tx+\mathbf{1}^TXd+\lambda_1\|X_*\|_1=m \}.$$
A simple option is  perhaps to choose $\cX=\Delta_x\times\Delta_X$, where $\Delta_x:=\{x\geq 0, \sum_ux_u\leq \frac{m}{T}\}$ and $\Delta_X:=\{X\geq 0, \sum_{u,u'}X_{uu'}\leq \frac{m}{\lambda}\}$. Seemingly, a good choice of the proximal setup is to equip the domain with with entropy distance generating function. We suggest to use the following proximal setups:
$$ \omega(u=[x,X;y])=\alpha_1\sum_{u} x_u\log(x_u)+\alpha_2\sum_{u,u'}X_{uu'}\log(X_{uu'})+\frac{1}{2}\|y\|_2^2$$

We present in Algorithm~\ref{alg:CMPforHawkes} the customized Composite Mirror Prox algorithm for the network estimation problem. 
\begin{algorithm}[H]
\caption{Composite Mirror Prox Algorithm for Problem (\ref{model:saddleHawkes})}
\label{alg:CMPforHawkes}
\begin{algorithmic}
\STATE Given $\alpha_1, \alpha_2>0$ and $\gamma_t>0$.
\STATE 0. Initialize $x^1\in \Delta_x$, $X^1\in\Delta_X$, $y^1_j=1/({e_j^Tx^1+\Tr(A_jX^1)}), j=1,2,\ldots,m$.
\FOR{$t=1,2,\ldots, T$}
  \STATE 1. Compute
\vspace{-2mm}
$$\begin{array}{rll}
\hat x^t_u&=x_u^t\exp\{-\gamma_t(T-\sum_{j}y_j^te^j_u)/\alpha_1\}, \; \forall u, \\
\hat X^t_{u,u'}&=X_{u,u'}^t\exp\{-\gamma_t(d_{u'}-\sum_{j} y_j^ta_{uu'}^j +\lambda_1)/\alpha_2\}, \; \forall u,u',\\
\hat y^t_j&= \frac{1}{2}\left(-\gamma_t(q_j-y^t_j)+\sqrt{\gamma_t^2(q_j-y^t_j)^2+4\gamma_t}\right), \forall j, \text{ where } q_j=e_j^Tx^t+\Tr(A_jX^t).
\end{array}$$

 \STATE 2. Compute
\vspace{-2mm}
$$\begin{array}{rll}
 x^{t+1}_u&=x_u^t\exp\{-\gamma_t(T-\sum_{j}\hat y_j^te^j_u)/\alpha_1\}, \; \forall u,  \\
 X^{t+1}_{u,u'}&=X_{u,u'}^t\exp\{-\gamma_t(d_{u'}-\sum_{j} \hat y_j^ta_{uu'}^j+\lambda_1)/\alpha_2\}, \\
 y^{t+1}_j&= \frac{1}{2}\left(-\gamma_t(q_j-y^t_j)+\sqrt{\gamma_t^2(q_j-y^t_j)^2+4\gamma_t}\right), \forall j, \text{ where } q_j=e_j^T\hat x^t+\Tr(A_j\hat X^t).
\end{array}$$

\ENDFOR\\
Output $x_{T} =\frac{1}{T}\sum_{t=1}^T \lambda_t x^t$
\end{algorithmic}
\end{algorithm}

\paragraph{Remark.} To avoid redundancy, we are not going to present the full algorithmic steps for the randomized block Mirror Prox algorithm. Note that the dual variables $y$ can be naturally divided into blocks that corresponds to the data points of each user. In our experiment, at each iteration, we randomly pick one user and use its data points to compute the gradient and proceed the update. The iteration computation cost reduces from $O(m)$ to $O(m/n)$, where $m$ is the total number of events, and $n$ is the number of users. We provide below the theoretical convergence rates for the three algorithms, Mirror Descent, composite Mirror Prox, and randomized block Mirror Prox. 
\begin{table}[H]
{\small
\caption{Convergence rates of different algorithms for Network Estimation}
\begin{tabular}{c|c|c|c|c|c}
\hline
optimization algorithm &type & guarantee &avg. iteration cost &  convergence &  constant\\
\hline\hline
MD \cite{ben2001ordered} &batch&  primal & $O(m)$ &$O(M/\sqrt{t})$ &  $M$ unbounded\\
\hline
CMP(this paper) & batch & primal and dual & $O(m/n)$ & $O(\cL/t)$ &  $\cL$ bounded\\
\hline
RB-CMP (this paper) &stoch. &sad. point gap & $O(m/n)$ & $O(\cL/t)$ &  $\cL$ bounded\\
\hline
\end{tabular}}
\end{table}

\end{appendix}

%% file: main-rpt.bbl
\begin{thebibliography}{10}

\bibitem{bach-et-al}
F.~Bach, R.~Jenatton, J.~Mairal, and G.~Obozinski.
\newblock Optimization with sparsity-inducing penalties.
\newblock {\em Foundations and Trends in Machine Learning}, 4(1):1--106, 2012.

\bibitem{Bauschke:2011}
H.~Bauschke and P.~Combettes.
\newblock {\em Convex Analysis and Monotone Operator Theory in Hilbert Spaces}.
\newblock Springer, 2011.

\bibitem{Teboulle09a}
A.~Beck and M.~Teboulle.
\newblock A fast iterative shrinkage-thresholding algorithm for linear inverse
  problems.
\newblock {\em SIAM Journal on Imaging Sciences}, 2(1):183--202, 2009.

\bibitem{ben2001ordered}
A.~Ben-Tal, T.~Margalit, and A.~Nemirovski.
\newblock The ordered subsets mirror descent optimization method with
  applications to tomography.
\newblock {\em SIAM Journal on Optimization}, 12(1):79--108, 2001.

\bibitem{bertsekas2011incremental}
D.~P. Bertsekas.
\newblock Incremental gradient, subgradient, and proximal methods for convex
  optimization: A survey.
\newblock {\em Optimization for Machine Learning}, 2010:1--38, 2011.

\bibitem{chambolle2011first}
A.~Chambolle and T.~Pock.
\newblock A first-order primal-dual algorithm for convex problems with
  applications to imaging.
\newblock {\em Journal of Mathematical Imaging and Vision}, 40(1):120--145,
  2011.

\bibitem{dang2014randomized}
C.~Dang and G.~Lan.
\newblock Randomized first-order methods for saddle point optimization.
\newblock {\em arXiv:1409.8625}, 2014.

\bibitem{dang2015sbmd}
C.~Dang and G.~Lan.
\newblock Stochastic block mirror descent methods for nonsmooth and stochastic optimization. \newblock {\em SIAM Journal on Optimization}, 25(2), pp.856-881.

\bibitem{Nan15}
N.~Du, Y.~Wang, N.~He, and L.~Song.
\newblock Time-sensitive recommendation from recurrent user activities.
\newblock In {\em NIPS}, 2015.

\bibitem{duchi2010composite}
J.~Duchi, S.~Shalev-Shwartz, Y.~Singer, and A.~Tewari.
\newblock Composite objective mirror descent.
\newblock In {\em COLT}, pages 14--26. Citeseer, 2010.

\bibitem{Green:Silverman:1994}
P.~Green and B.~Silverman.
\newblock { Nonparametric regression and generalized linear models : a
  roughness penalty approach}.
\newblock {\em Monographs on statistics and applied probability},
  1994.

\bibitem{gunawardana2011model}
A.~Gunawardana, C.~Meek, and P.~Xu.
\newblock A model for temporal dependencies in event streams.
\newblock In {\em NIPS}, pages
  1962--1970, 2011.

\bibitem{harmany2012spiral}
Z.~Harmany, R.~Marcia, and R.~Willett.
\newblock This is spiral-tap: sparse poisson intensity reconstruction
  algorithms—theory and practice.
\newblock {\em Image Processing, IEEE Transactions on}, 21(3):1084--1096, 2012.

\bibitem{hawkes1971spectra}
A.~G. Hawkes.
\newblock Spectra of some self-exciting and mutually exciting point processes.
\newblock {\em Biometrika},1971.

\bibitem{he:harch:2015}
N.~He and Z.~Harchaoui.
\newblock Semi-proximal mirror-prox for nonsmooth composite minimization.
\newblock In {\em NIPS}, 2015.

\bibitem{he2015mirror}
N.~He, A.~Juditsky, and A.~Nemirovski.
\newblock Mirror prox algorithm for multi-term composite minimization and
  semi-separable problems.
\newblock {\em Computational Optimization and Applications}, 61(2):275--319,
  2015.

\bibitem{he2016accelerated}
Y.~He and R.~D. Monteiro.
\newblock An accelerated hpe-type algorithm for a class of composite
  convex-concave saddle-point problems.
\newblock {\em SIAM Journal on Optimization}, 26(1):29--56, 2016.

\bibitem{iwata2013discovering}
T.~Iwata, A.~Shah, and Z.~Ghahramani.
\newblock Discovering latent influence in online social activities via shared
  cascade poisson processes.
\newblock In {\em ACM SIGKDD}, 2013.

\bibitem{KapSubKarSriJaiSch15}
K.~Kapoor, K.~Subbian, J.~Srivastava, and P.~Schrater.
\newblock Just in time recommendations: Modeling the dynamics of boredom in
  activity streams.
\newblock In {\em WSDM }, 2015.

\bibitem{lu2013complexity}
Z.~Lu and L.~Xiao.
\newblock On the complexity analysis of randomized block-coordinate descent
  methods.
\newblock {\em Mathematical Programming}, pages 1--28, 2013.

\bibitem{MohShoMarBraetal11}
G.~Mohler, M.~Short, P.~Brantingham, F.~ Schoenberg, and G.~Tita.
\newblock Self-exciting point process modeling of crime.
\newblock {\em Journal of the American Statistical Association}, 106(493),
  2011.

\bibitem{Nem09}
A.~Nemirovski, A.~Juditsky, G.~Lan, and A.~Shapiro.
\newblock Robust stochastic approximation approach to stochastic programming.
\newblock {\em SIAM Journal on Optimization}, 19(4):1574--1609, 2009.

\bibitem{blair1985problem}
A.~S. Nemirovsky and D.~B. Yudin.
\newblock Problem complexity and method efficiency in optimization.
\newblock {\em Discrete Mathematics}, 1983.

\bibitem{nesterov2012efficiency}
Y.~Nesterov.
\newblock Efficiency of coordinate descent methods on huge-scale optimization
  problems.
\newblock {\em SIAM Journal on Optimization}, 22(2):341--362, 2012.

\bibitem{NesCompMin}
Y.~Nesterov.
\newblock Gradient methods for minimizing composite functions.
\newblock {\em Mathematical Programming}, 140(1):125--161, 2013.

\bibitem{rajaram2005poisson}
S.~Rajaram, T.~Graepel, and R.~Herbrich.
\newblock Poisson-networks: A model for structured point processes.
\newblock In {\em AISTATS}, 2005.

\bibitem{richtarik2014iteration}
P.~Richt{\'a}rik and M.~Tak{\'a}{\v{c}}.
\newblock Iteration complexity of randomized block-coordinate descent methods
  for minimizing a composite function.
\newblock {\em Mathematical Programming}, 144(1-2):1--38, 2014.

\bibitem{schmidt2013minimizing}
M.~Schmidt, N.~L. Roux, and F.~Bach.
\newblock Minimizing finite sums with the stochastic average gradient.
\newblock {\em arXiv:1309.2388}, 2013.

\bibitem{Shalev-Shwartz:2013}
S.~Shalev-Shwartz and T.~Zhang.
\newblock Stochastic dual coordinate ascent methods for regularized loss.
\newblock {\em J. Mach. Learn. Res.}, 14(1):567--599, 2013.

\bibitem{simma2012modeling}
A.~Simma and M.~I. Jordan.
\newblock Modeling events with cascades of poisson processes.
\newblock {\em arXiv:1203.3516}, 2012.

\bibitem{sra2008non}
S.~Sra, D.~Kim, and B.~Sch{\"o}lkopf.
\newblock Non-monotonic poisson likelihood maximization.
\newblock {\em Technical report},2008.

\bibitem{tran2013composite}
Q.~Tran-Dinh, A.~Kyrillidis, and V.~Cevher.
\newblock Composite self-concordant minimization.
\newblock {\em  arXiv:1308.2867}, 2013.

\bibitem{Tseng08}
P.~Tseng.
\newblock On accelerated proximal gradient methods for convex-concave
  optimization.
\newblock {\em submitted to SIAM Journal on Optimization}, 2009.

\bibitem{yanez2014primal}
F.~Yanez and F.~Bach.
\newblock Primal-dual algorithms for non-negative matrix factorization with the
  kullback-leibler divergence.
\newblock {\em arXiv:1412.1788}, 2014.

\bibitem{yang2015efficient}
T.~Yang, M.~Mahdavi, R.~Jin, and S.~Zhu.
\newblock An efficient primal dual prox method for non-smooth optimization.
\newblock {\em Machine Learning}, 98(3):369--406, 2015.

\bibitem{zhang2015stochastic}
Y.~Zhang, and L.~Xiao.
\newblock Stochastic primal-dual coordinate method for regularized empirical risk minimization. 
\newblock In {\em ICML},  2015.

\bibitem{zhou2013learning}
K.~Zhou, H.~Zha, and L.~Song.
\newblock Learning social infectivity in sparse low-rank networks using
  multi-dimensional hawkes processes.
\newblock In {\em AISTATS}, 2013.

\end{thebibliography}
